\documentclass{article}

\usepackage{microtype}
\usepackage{graphicx}
\usepackage{subfigure}
\usepackage{booktabs} 

\usepackage[page]{appendix}
\usepackage{latexsym,amsfonts,amssymb,amsmath,graphicx,epsf,bbm,float}
\usepackage{caption}
\usepackage{comment}
\usepackage{amsthm}
\usepackage{mathtools}
\usepackage{enumitem}

\newtheorem{theorem}{Theorem}
\newtheorem{lemma}{Lemma}

\newtheorem{proposition}{Proposition}

\newtheorem{definition}{Definition}

\usepackage{hyperref}

\usepackage[accepted]{icml2018}

\icmltitlerunning{Budgeted Experiment Design for Causal Structure Learning}

\begin{document}

\twocolumn[
\icmltitle{Budgeted Experiment Design for Causal Structure Learning}



%
%
%

\begin{icmlauthorlist}
\icmlauthor{AmirEmad Ghassami}{1}
\icmlauthor{Saber Salehkaleybar}{2}
\icmlauthor{Negar Kiyavash}{1}
\icmlauthor{Elias Bareinboim}{3}
\end{icmlauthorlist}

\icmlaffiliation{1}{Department of ECE, and Coordinated Science Laboratory, University of Illinois at Urbana-Champaign, Urbana, IL, USA}
\icmlaffiliation{2}{Department of Electrical Engineering, Sharif University of Technology, Tehran, Iran}
\icmlaffiliation{3}{Department of Computer Science, Purdue University, West Lafayette, IN, USA}

\icmlcorrespondingauthor{AmirEmad Ghassami}{ghassam2@illinois.edu}

\icmlkeywords{Causal inference, causal structure learning, Causal experiment design}

\vskip 0.3in
]



\printAffiliationsAndNotice{}  


\begin{abstract}
We study the problem of causal structure learning when the experimenter is limited to perform at most $k$ non-adaptive experiments of size $1$. We formulate the problem of finding the best intervention target set as an optimization problem, which aims to maximize the average number of edges whose directions are resolved. We prove that the corresponding objective function is submodular and a greedy algorithm suffices to achieve $(1-\frac{1}{e})$-approximation of the optimal value. We further present an accelerated variant of the greedy algorithm, which can lead to orders of magnitude performance speedup. We validate our proposed approach on synthetic and real graphs. The results show that compared to the purely observational setting, our algorithm orients the majority of the edges through a considerably small number of interventions.
\end{abstract}

\vspace{-5mm}
\section{Introduction}
\label{sec:intro}

The problem of learning the causal relations underlying a complex system is of great interest in AI and throughout the empirical sciences. 
Causal systems are commonly represented by directed acyclic graphs (DAGs), where the vertices are random variables and an edge from variable $X$ to $Y$ indicates that variable $X$ is a direct cause of $Y$ \cite{pearl2009causality,spirtes2000causation,bareinboim2016causal}.

To uncover the causal relations among a set of variables, if restricted to work with only observational data from the variables, one can use a constraint-based algorithm such as IC, IC$^*$ \cite{pearl2009causality}, and PC, FCI \cite{spirtes2000causation}, or a score-based methods, including \cite{meek1997graphical,chickering2002optimal,tian2012bayesian,solus2017consistency}. Such purely observational approaches reconstruct the causal graph up to Markov equivalence class, and hence, the investigator is commonly left with some (or in some cases many) unresolved causal relations.
Albeit, under some extra assumptions, 
full structure learning using merely observational data is feasible \cite{mooij2016distinguishing,shimizu2006linear,hoyer2009nonlinear,peters2012identifiability,ghassami2017learning,quinn2015directed,sun2015causal}.

 On the other hand, it is well-understood that whenever the investigator can perform sufficient number of interventions, the causal graph representing the underlying system can be fully recovered. 
There is a growing body of research on learning causal structures using interventional data in causally sufficient \cite{eberhardt2007causation,hauser2012characterization,cooper1999causal,he2008active}, and causally insufficient systems (i.e., with latent confounders) \cite{kocaoglu2017experimental}. 
An interventional structure learning approach requires performing a set of experiments, each intervening on a subset of the variables, and subsequently collecting data from the intervened system.
In this setting, two natural questions arise:
\vspace{-3mm}
\begin{enumerate}[noitemsep]
\item What is the smallest required number of experiments in order to fully learn the underlying causal graph?
\item For a fixed number of experiments (budget), what portion of the causal graph is learnable?
\vspace{-3mm}
\end{enumerate}

The first problem has been addressed in the literature under different assumptions.
\cite{eberhardt2005number} obtained the worst case bounds on the number of required experiments.
Connections between the experiment design problem and the problem of 
finding a separating system in a graph was studied in \cite{eberhardt2007causation,hyttinen2013experiment}.
Adaptive algorithms for experiment design were proposed in \cite{hauser2014two,shanmugam2015learning}.
In \cite{shanmugam2015learning}, the authors present information-theoretic lower bounds on the number of required experiments for both deterministic and randomized adaptive approaches. 
In \cite{kocaoglu2017cost}, the authors considered costs for intervening on each variable and derived an experiment design algorithm with minimum total cost that reconstructs the whole structure.

The second question mentioned above has received less attention. We address the second question herein. Specifically, we consider a setup with budget limitation of $k$ experiments. In some applications (for instance, reconstructing gene regulatory networks from knockout data in biology), performing simultaneous interventions on multiple variables is not always feasible. To insure that our model is applicable to such settings, we set each experiment to contain exactly one intervention. 
Even when more interventions are allowed per experiment, this constraint allows us to lower bound the fraction of the causal graph that is learnable. 
This is a distinctive feature of our work since most of the literature assumes that the size of each experiment is larger than one, in some cases going as high as half of the number of variables. 
Note that our results cannot be designed by limiting the size of experiments in those approaches to one, as that would result in trivial designs for experiments, such as requiring to intervene on all the variables. 
The authors of \cite{kocaoglu2017cost} also considered the case in which the number of experiments is limited. However, each experiment in their setup is allowed to include intervening on arbitrary number of variables. 
We put the budget restriction on the number of variables to be randomized, i.e., the number of interventions (as opposed to the number of experiments), which we believe is a more natural budgeting constraint\footnote{One less usual connection is with the literature concerned with the 
influence maximization problem.
The goal in the influence maximization problem is to find $k$ vertices (seeds) in a given network such that under a specified influence model, the expected number of vertices influenced by the seeds is maximized \cite{kempe2003maximizing,leskovec2007cost,chen2009efficient}. 
Besides the interpretative differences, an important distinction between the two problems is that in influence maximization  problem, the goal is to spread the influence to the vertices of the graph, while in budgeted experiment design problem, the goal is to pick the initial $k$ vertices in a way that leads to discovering the orientation of as many edges as possible.  Therefore, the optimal solution to these two problems for a given graph can be quite different (see the supplementary materials for an example).}.

\textbf{Contributions.}
In our interventional structure learning algorithm, first an observational test, such as PC algorithm \cite{spirtes2000causation}, is performed on the set of variables. This test learns the skeleton as well as the orientation of some of the edges of the causal graph. Based on the result of the initial observational stage, the complete set of $k$ experiments is designed.
The more formal description of the problem statement is provided in Section \ref{sec:problem}. Our main contributions are summarized bellow:
\begin{itemize}[noitemsep]
\vspace{-4mm}
\item We cast the problem of finding the best intervention target set as an optimization problem which aims to maximize the average number of edges whose directions are resolved. 
\item We prove that the corresponding objective function is monotonically increasing and submodular. This implies that a general greedy algorithm is a $(1-\frac{1}{e})$-approximation algorithm.
\item Since computing the objective function is, in general, intractable for a given set, we propose an unbiased estimator of this function, which for graphs with bounded degree, has a polynomial time complexity. 
For graphs with high degree,
we provide another efficient, albeit slightly biased estimator.
\vspace{-4mm}
\end{itemize}
We implement an accelerated variant of the general greedy algorithm through {\it lazy} evaluations, originally proposed by Minoux \cite{minoux1978accelerated}. This algorithm can lead to orders of magnitude performance speedup.
Using synthetic and real data, in Section \ref{sec:experiments}, we show that the proposed approach recovers a significant portion of the edges by performing only a few interventions in the underlying causal system.

\vspace{-2mm}
\section{Problem Description}
\label{sec:problem}
We denote an undirected graph with a pair $G=(V,E)$, and a directed graph with a pair $G=(V,A)$, where $V$ is the vertex set, and $E$ and $A$ are sets of undirected and directed edges, respectively. A mixed graph $G=(V,E,A)$, comprises both undirected and directed edges.

We use the language of Structural Causal Models (SCM) \cite{pearl2009causality}. Formally, a SCM is a 4-tuple $\langle U, V, F, P(U)\rangle$, where $U$ is a set of exogenous (latent) variables and $V$ is a set of endogenous (measured) variables. $F$ represents a collection of functions $F = \{f_v\}$ such that each endogenous variable $X_v$ is determined by a function $f_v \in F$, where $f_{v}$ is a mapping from the respective domain of $U_{v}\cup \mathit{PA}_{v}$ to $X_{v}$, where $U_{v}\subseteq U$, $\mathit{PA}_{v} \subseteq V \backslash X_{v}$, and the entire set $F$ forms a mapping from $U$ to $V$. The uncertainty is encoded through a probability distribution over the exogenous variables, $P(U)$. Each SCM induces a causal graph $G$, where vertex $v$ corresponds to endogenous variable $X_v$, and the arguments of the functions correspond to its parent set. We will refer to variables and their corresponding vertices interchangeably. 
We consider causally sufficient systems in which the exogenous variables are independent. We also assume the observational and experimental distributions are faithful \cite{spirtes2000causation}. For a detailed discussion on the properties of structural models, we refer readers to \cite{pearl2009causality}. 
We introduce next other definitions that will be used throughout the paper. 
\begin{definition}
Two causal DAGs $G_1$ and $G_2$ over $V$ are Markov equivalent if they represent the same set of conditional independence assertions. The set of all graphs over $V$ is partitioned into a set of mutually exclusive and exhaustive Markov equivalence classes \cite{koller2009probabilistic}.
\end{definition}
\vspace{-1mm}
\begin{definition}
The essential graph of $G$, denoted by ${Ess}(G)$, is a mixed graph 
in which the directed edges are those that have the same direction in all elements of the Markov equivalence class of $G$, and the undirected edges are those whose direction differ in at least two elements of the class \cite{andersson1997characterization}.
\end{definition}
\vspace{-1mm}
\begin{definition} 
A complete conditional independence-based (CCI) algorithm is an observational structure learning algorithm resulting in learning the Markov equivalence class of the ground truth DAG.
\end{definition}
\vspace{-2mm}
We denote the underlying true causal structure (ground truth DAG) by $G^*$.
After performing CCI on the data from $G^*$, we denote the revealed set of directed edges by $A(Ess(G^*))$ and the set of undirected edges by $E(Ess(G^*))$.
In general, the size of Markov equivalence classes can be very large. For instance, as observed in \cite{he2015counting}, the sizes of classes with even sparse essential graphs grow approximately exponentially with the order of the graph.
Interventions will allow us to differentiate among the different causal structures within a Markov equivalence class.

We use the same notion of ideal intervention as in \cite{eberhardt2005number,pearl2009causality}. For an intervention on variable $X$, denoted by $\textit{do}(X)$, the influence of all the variables on the target variable $X$ is removed, and $X$ is randomized by forcing values from an independent distribution
on it. This intervention changes the joint distribution of all variables in the system for which $X$ is a direct or indirect cause, and results in interventional joint distribution $P(V|\textit{do}(X))$. In our terminology, an intervention is always on a single variable. An interventional structure learning algorithm consists of a set of $k$ experiments\footnote{Note that in most of the related work, each intervention is what we refer to as experiment here. That is, each intervention can contain many variables that are randomized simultaneously, while in our terminology, as intervention is always on a single variable.} $\mathcal{E}=\{\mathcal{E}_1, \mathcal{E}_2, ..., \mathcal{E}_k\}$, where each experiment $\mathcal{E}_i$ contains $m_i$ interventions, which are performed simultaneously, i.e., $\mathcal{E}_i=\{v_1^{(i)},v_2^{(i)}, ...,v_{m_i}^{(i)}\}$, for $1\le i \le k$. More precisely, in experiment $\mathcal{E}_i$ data is drawn from distribution $P(V|\textit{do}(X_{v_1^{(i)}}, X_{v_2^{(i)}}, ..., X_{v_{m_i}^{(i)}}))$.

A structure learning algorithm may be adaptive, in which case the experiments are performed sequentially and the information obtained from the previous experiments is used to design the next one, or passive, in which case all the experiments are designed in one shot.
The approach we take here is to first perform a CCI algorithm to learn the skeleton and the direction of the edges in $A(Ess(G^*))$, and then design the experiments in a passive manner\footnote{This approach is referred to as the passive setup by \cite{shanmugam2015learning}, while \cite{eberhardt2005number} uses the term passive for a setting in which the interventions are selected without performing the null experiment.} under a budget constraint.
This approach gives the experimenter the ability to perform the experiments in parallel without the need to wait for the result of one experiment to choose the next one.
For example, in the study of gene regulatory networks (GRNs), when the GRN of all cells are the same, experiments can be performed simultaneously on different cells.

We consider a setup in which the number of interventions is fixed at $k$, and seek to design a set of experiments that allows learning the directions of as many edges in $E(Ess(G^*))$ as possible.
We focus on single-intervention experiment setup in which for all experiments,  $m_i=1$. Therefore, our experiment set is of the form $\mathcal{E}=\{\{v_1^{(1)}\}, \{v_1^{(2)}\}, ..., \{v_1^{(k)}\}\}$. To simplify the notation, we denote the set of intervention targets as $\mathcal{I}=\{v_1,...,v_k\}$, where $v_i=v_1^{(i)}$.
As shown in \cite{eberhardt2007causation,hyttinen2013experiment}, observing the result of the null experiment, allows orientating the edge $\{u,v\}\in E(Ess(G^*))$, if there exists $\mathcal{E}_l\in\mathcal{E}$ such that
$
(u\in\mathcal{E}_l\textit{ , }v\notin\mathcal{E}_l)\text{ or }
(v\in\mathcal{E}_l\textit{ , }u\notin\mathcal{E}_l).
$
On the other hand, if for all experiments $\mathcal{E}_l\in\mathcal{E}$ both
$u\in\mathcal{E}_l\text{ and }v\in\mathcal{E}_l$, the orientation of $\{u,v\}$ cannot be learned. An experiment in which both $u$ and $v$ are intervened on is called a zero information experiment for $u$ and $v$. Our setup in which $m_i=1$, for all $i\in\{1,...,k\}$, avoids such zero information experiments.
Specifically, we focus on the following problem: \textit{If the experimenter is allowed to perform $k$ experiments, each of size $1$, what portion of the graph could, on average, be reconstructed?}
We formalize the problem statement in the rest of this section.

As discussed earlier, we only intervene on a single variable in each experiment. Hence, due to avoiding the issue of zero information experiments \cite{eberhardt2005number}, an experiment with intervention on vertex $v$ will lead to learning the orientation of all the edges intersecting with $v$. Therefore, the entire experiment set $\mathcal{I}\subseteq V$ leads to learning the orientation of all the edges intersecting with members of $\mathcal{I}$. We denote the set of these learned directed edges by $A_{G^*}^{(\mathcal{I})}$ (which clearly depends on the structure of the true DAG $G^*$).
Note that after learning the edges in $A_{G^*}^{(\mathcal{I})}$, we could also possibly resolve the direction of more edges by applying the Meek rules \cite{verma1992algorithm, meek1997graphical} to $A_{G^*}^{(\mathcal{I})}\cup A(Ess(G^*))$, set of all directed edges after intervention $\mathcal{I}$.
Let $H=(V(H),E(H))$ denote the undirected subgraph of $Ess(G^*)$. For any given set of directed edges $\mathcal{A}$ from the true DAG $G^*$, define $R(\mathcal{A},G^*)$ as the subset of $E(H)$, whose directions can be learned by applying Meek rules starting from the set of directed edges $\mathcal{A}\cup A(Ess(G^*))$.
Using this notation, experiment $\mathcal{I}$ results in learning the direction of edges in $R(A_{G^*}^{(\mathcal{I})},G^*)$. 

Let $D(\mathcal{I},G^*)=|R(A_{G^*}^{(\mathcal{I})},G^*)|$, i.e., the cardinality of $R(A_{G^*}^{(\mathcal{I})},G^*)$, and let $\mathcal{G}$ denote the set of all DAGs in the Markov equivalence class of $G^*$.
As  we do not know the ground truth DAG, and since there is no preference between the members of the Markov equivalence class, $G^*$ is equally likely to be any of the DAGs in $\mathcal{G}$. Hence, the expected number of the edges recovered through the experiment $\mathcal{I}$ is
\vspace{-1mm}
\begin{equation}
\label{eq:summ}
\begin{aligned}
\mathcal{D}(\mathcal{I})\coloneqq\mathbb{E}_{G_i}[D(\mathcal{I},G_i)]
=\frac{1}{|\mathcal{G}|}\sum_{G_i\in\mathcal{G}}D(\mathcal{I},G_i).
\end{aligned}
\vspace{-2mm}
\end{equation} 
Thus, our problem of interest can be formulated as finding some intervention target set $\mathcal{I}\subseteq V$ of cardinality $k$ that maximizes $\mathcal{D}(\mathcal{I})$:
\vspace{-2mm}
\begin{equation}
\label{eq:optproblem}
\max_{\mathcal{I}:\mathcal{I}\subseteq V} \mathcal{D}(\mathcal{I})~~~\text{s.t.}~~~|\mathcal{I}|= k.
\vspace{-2mm}
\end{equation} 
This is a challenging optimization problem for two reasons: First, finding an optimal $\mathcal{I}$ requires a combinatorial search. Second, even for a given set $\mathcal{I}$, computing $\mathcal{D}(\mathcal{I})$ when the value of $k$ or the cardinality of the Markov equivalence class is large,
can be computationally intractable.


\vspace{-2mm}
\section{Proposed Approach}
\label{sec:approach}
We start by defining monotonicity and submodularity properties for a set function.
\begin{definition}
A set function $f: 2^V \to \mathbb{R}$ is monotonically increasing if for all sets $\mathcal{I}_1\subseteq \mathcal{I}_2\subseteq V$, we have
$
f(\mathcal{I}_1)\le f(\mathcal{I}_2)
$.
\end{definition}
\begin{definition}
A set function $f: 2^V \to \mathbb{R}$ is submodular if for all subsets $\mathcal{I}_1\subseteq \mathcal{I}_2\subseteq V$ and all $v \in V \setminus \mathcal{I}_2$\footnote{If $f$ is monotonically increasing, the definition relaxes to $v\in V$.},
\[
f(\mathcal{I}_1 \cup \{ v \}) - f(\mathcal{I}_1) \ge f(\mathcal{I}_2 \cup \{ v \}) - f(\mathcal{I}_2).
\]
\end{definition}
\vspace{-2mm}
Nemhauser et al. showed that if $f$ is a submodular, monotonically increasing set function with $f(\emptyset)=0$, then the set $\hat{\mathcal{I}}$ with $|\hat{\mathcal{I}}|=k$ found by the greedy algorithm satisfies $f(\hat{\mathcal{I}})\ge(1-\frac{1}{e})\max_{\mathcal{I}:|\mathcal{I}|=k}f(\mathcal{I})$ \cite{nemhauser1978analysis}. That is, the greedy algorithm is a $(1-\frac{1}{e})$-approximation algorithm.
We will use this result in our proposed approach. 

We will show that the set function $\mathcal{D}$ is monotonically increasing and submodular, and hence, the greedy algorithm is a $(1-\frac{1}{e})$-approximation algorithm to the maximization problem \eqref{eq:optproblem}.
\begin{lemma}
\label{lem:mono}
The set function $\mathcal{D}$ defined in \eqref{eq:summ} is monotonically increasing, i.e., for sets $\mathcal{I}_1\subseteq \mathcal{I}_2$, we have
\vspace{-2mm}
\[
\mathcal{D}(\mathcal{I}_2)\le \mathcal{D}(\mathcal{I}_1).
\]
\end{lemma}
\vspace{-2mm}
See the supplementary materials for the proof.

The following lemma plays a fundamental role in the proof of submodularity of the set function $\mathcal{D}$. 
\begin{lemma}   
\label{lem:nofusion}
For sets $\mathcal{I}_1,\mathcal{I}_2\subseteq V$,
\vspace{-2mm}
\[
R(A^{(\mathcal{I}_1\cup \mathcal{I}_2)}_{G^*},G^*)=R(A^{(\mathcal{I}_1)}_{G^*},G^*)\cup R(A^{(\mathcal{I}_2)}_{G^*},G^*).
\]
\end{lemma}
\vspace{-2mm}
See the supplementary materials for the proof.

Interpreting $R(A^{(\mathcal{I})}_{G^*},G^*)$ as the information obtained by intervening on set $\mathcal{I}$, Lemma \ref{lem:nofusion} indicates that the result of two simultaneous interventions does not generate any new information which was not provided by the union of the information of each of the interventions.

\begin{theorem}
\label{thm:submodular}
The set function $\mathcal{D}$ defined in \eqref{eq:summ} is a submodular function.
\end{theorem}
\vspace{-5mm}
\begin{proof}
Due to Lemma \ref{lem:mono}, it suffices to show that for $\mathcal{I}_1\subseteq \mathcal{I}_2\subseteq V$, and $v\in V$, we have
$
\mathcal{D}(\mathcal{I}_1\cup \{v\})-\mathcal{D}(\mathcal{I}_1)\ge \mathcal{D}(\mathcal{I}_2\cup \{v\})-{\mathcal{D}}(\mathcal{I}_2)$.
First we show that for a given directed graph $G_i\in\mathcal{G}$ the function $D(\mathcal{I},G_i)$ is a submodular function of $\mathcal{I}$. 
From Lemma \ref{lem:nofusion}, we have $R(A^{(\mathcal{I}_1\cup \{v\})}_{G_i},G_i)=R(A^{(\mathcal{I}_1)}_{G_i},G_i)\cup R(A^{(\{v\})}_{G_i},G_i)$. Therefore,
\vspace{-2mm}
\begin{align*}
D(\mathcal{I}_1&\cup \{v\},G_i)-D(\mathcal{I}_1,G_i)\\
&=|R(A^{(\mathcal{I}_1\cup \{v\})}_{G_i},G_i)|-|R(A^{(\mathcal{I}_1)}_{G_i},G_i)|\\
&=|R(A^{(\mathcal{I}_1)}_{G_i},G_i)\cup R(A^{(\{v\})}_{G_i},G_i)|-|R(A^{(\mathcal{I}_1)}_{G_i},G_i)|\\
&=|R(A^{( \{v\})}_{G_i},G_i)|-|R(A^{(\mathcal{I}_1)}_{G_i},G_i)\cap R(A^{(\{v\})}_{G_i},G_i)|.
\end{align*}
\vspace{-2mm}
Similarly, 
\begin{align*}
D&(\mathcal{I}_2\cup \{v\},G_i)-D(\mathcal{I}_2,G_i)\\
&=|R(A^{( \{v\})}_{G_i},G_i)|-|R(A^{(\mathcal{I}_2)}_{G_i},G_i)\cap R(A^{(\{v\})}_{G_i},G_i)|.
\end{align*}
Since $\mathcal{I}_1\subseteq \mathcal{I}_2$, as seen in the proof of Lemma \ref{lem:mono}, $R(A^{(\mathcal{I}_1)}_{G_i},G_i)\subseteq R(A^{(\mathcal{I}_2)}_{G_i},G_i)$. Therefore, 
$-|R(A^{(\mathcal{I}_1)}_{G_i},G_i)$ $\cap R(A^{(\{v\})}_{G_i},G_i)|
\ge-|R(A^{(\mathcal{I}_2)}_{G_i},G_i)\cap R(A^{(\{v\})}_{G_i},G_i)|$,
which implies that
\[
D(\mathcal{I}_1\cup \{v\},G_i)-D(\mathcal{I}_1,G_i)\ge D(\mathcal{I}_2\cup \{v\},G_i)-D(\mathcal{I}_2,G_i).
\] 
This together with the fact that the function $D(\mathcal{I},G_i)$ is a monotonically increasing function of $\mathcal{I}$ (observed in the proof of Lemma \ref{lem:mono}) shows that $D(\mathcal{I},G_i)$ is a submodular function of $\mathcal{I}$.
Finally, we have $\mathcal{D}(\mathcal{I})=\frac{1}{|\mathcal{G}|}\sum_{G_i\in\mathcal{G}}D(\mathcal{I},G_i)$.
Since a non-negative linear combination of submodular functions is also submodular, the proof is concluded.
\end{proof}

\begin{algorithm}[t]
\begin{algorithmic}
 \STATE {\bf Input:} Joint distribution over $V$, and budget $k$.
\STATE Obtain $Ess(G^*)$ by performing a CCI algorithm.
 \STATE {\bf initiate:}  $\mathcal{I}_0=\emptyset$
\FOR{$i=1$ to $k$}
\STATE $v_i=\arg\max_{v\in V\backslash \mathcal{I}_{i-1}}\hat{\mathcal{D}}(\mathcal{I}_{i-1}\cup \{v\})-\hat{\mathcal{D}}(\mathcal{I}_{i-1})$
\STATE $\mathcal{I}_i=\mathcal{I}_{i-1}\cup\{v_i\}$
\ENDFOR
\STATE {\bf Output:} $\hat{\mathcal{I}}=\mathcal{I}_k$
 \caption{General Greedy Algorithm}
 \label{algorithm:GG}
\end{algorithmic}
\end{algorithm}
\vspace{-2mm}
Our General Greedy Algorithm is presented in Algorithm \ref{algorithm:GG}. Define the marginal gain of variable $v$ when the previous chosen set is $\mathcal{I}$ as $ \Delta_v(\mathcal{I})=\mathcal{D}(\mathcal{I}\cup \{v\})-\mathcal{D}(\mathcal{I})$. The greedy algorithm iteratively adds a variable which has the largest marginal gain to the intervention target set until it runs out of budget. However, as mentioned in Section \ref{sec:problem}, another issue regarding solving the optimization problem \eqref{eq:optproblem} is the computational intractability of calculating $\mathcal{D}(\mathcal{I})$ for a given intervention target set $\mathcal{I}$. We propose running Monte-Carlo simulations of the intervention model for sufficiently large number of times to obtain an accurate estimation of $\mathcal{D}(\mathcal{I})$. The pseudo-code of our estimator is presented in Subroutine 1. 
In this subroutine, for the given essential graph $Ess(G^*)$, we generate $N$ DAGs from the Markov equivalence class of $G^*$.
The generated DAGs are kept in a multiset $\mathcal{G}'$.
Note that $\mathcal{G'}$ is a multiset in which repetition is allowed, and operator $\uplus$ in the pseudo-code indicates the multiset addition.
Finally, we calculate the estimated value $\hat{\mathcal{D}}(\mathcal{I})$ on $\mathcal{G}'$ instead of $\mathcal{G}$ as $\hat{\mathcal{D}}(\mathcal{I})=\frac{1}{|\mathcal{G}'|}\sum_{G'_i\in\mathcal{G}'}D(\mathcal{I},G'_i)$. 

 We use the unbiased sampler introduced in \cite{ghassami2018mec} to generate uniform samples from the Markov equivalence class. 
For the sake of completeness of the presentation, we briefly describe the idea in this sampler:
Let $H=(V(H),E(H))$ denote the undirected subgraph of $Ess(G^*)$. Note that in general, $H$ can be disconnected, with the set of its components denoted by $\textit{comp}(H)$. Note that each of these components is an essential graph. Chickering showed that learning the direction of any edge in one component of $H$ will not reveal any information about the direction of edges in the other components \cite{chickering2002optimal}.
Therefore, we can orient the edges in each component independently.
All the members of a Markov equivalence class agree on v-structures\footnote{A v-structure is a structure containing two converging directed edges whose tails are not connected by an edge.} \cite{judea1991equivalence}.
Therefore, since the essential graphs corresponding to the components of $H$ are undirected, all the members of their corresponding class should be v-structure-free. Therefore, in each member, there is only one source vertex\footnote{A source vertex has incoming degree equal to zero.}, and
once a source is determined, an edge could be oriented as long as its endpoints are not at equal distance from the source \cite{ghassami2018mec}. 
This is achieved by function $\textsc{Flowed}(v,G)$ in the algorithm, where $G$ is a connected undirected essential graph and $v$ is the source vertex. Function $\textsc{W}$ uses \textsc{Flowed} to find $\textsc{W}(v,G)$, which is the number of members of the class in which $v$ is the source vertex. This calculation is done in a recursive manner (see \cite{ghassami2018mec} for the details). Finally, function \textsc{RandEdge} sets a vertex $v^*$ as the source vertex with probability proportional to $\textsc{W}(v^*,G)$, and saves the direction of resolved edges in set $A$ until all the edges are directed. 

The used sampler satisfies $P(G'=G)=1/|\mathcal{G}|$ \cite{ghassami2018mec}; hence for any $\mathcal{I}\subseteq V$,
$\hat{\mathcal{D}}(\mathcal{I})$ obtained from Subroutine 1 is an unbiased estimate of $\mathcal{D}(\mathcal{I})$, i.e., $\mathbb{E}[\hat{\mathcal{D}}(\mathcal{I})]=\mathcal{D}(\mathcal{I})$.
To show the unbiasedness, suppose $G'$ is a random generated graph in the subroutine. The result is immediate from the fact that
\vspace{-1mm}
\begin{align*}
\mathbb{E}[D(\mathcal{I},G')]&=\sum_{G\in\mathcal{G}}P(G'=G)D(\mathcal{I},G)\\
&=\frac{1}{|\mathcal{G}|}\sum_{G\in\mathcal{G}}D(\mathcal{I},G)=\mathcal{D}(\mathcal{I}).
\end{align*}

Since we use random sampling in a greedy algorithm, we term our proposed approach the \textit{Random Greedy Intervention Design} (Ran-GrID).

Next we consider the required cardinality of the set $\mathcal{G}'$ to obtain a desired accuracy in estimating $\mathcal{D}(\mathcal{I})$. We use Chernoff bound for this purpose. 
\begin{proposition}[\textbf{Chernoff Bound}]
\label{prop:Chbound}
Let $X_1,...,X_N$ be independent random variables such that for all $i$, $0 \le X_i \le 1$. Let $\mu=\mathbb{E}[\sum_{i=1}^NX_i]$. Then
\vspace{-2mm}
\[
P(|\sum_{i=1}^NX_i - \mu |\ge\epsilon\mu) \le 2 \exp (-\frac{\epsilon^2}{2+\epsilon}\mu).
\]
\end{proposition}
\vspace{-2mm}

\begin{algorithm}[t]
\begin{algorithmic}
\STATE {\bf Input:} $Ess(G^*)$, intervention target set $\mathcal{I}$, and $N$.
\STATE {\bf initiate:}  $\mathcal{G'}=\emptyset$
\FOR{$i=1$ to $N$,}
\STATE $ A=A(Ess(G^*))\bigcup_{\hat{H}\in\textit{comp}(H)}\textsc{RandEdge}(\hat{H},\emptyset)$.
\STATE $G'_i=(V(Ess(G^*),A)$.
\STATE $\mathcal{G'}=\mathcal{G'}\uplus G'_i$
\ENDFOR
\STATE {\bf Output:} $\hat{\mathcal{D}}(\mathcal{I})=\frac{1}{|\mathcal{G}'|}\sum_{G'_i\in\mathcal{G}'}D(\mathcal{I},G'_i)$
\\ \vspace{-2mm} \hrulefill
\STATE {\bf function} \textsc{Flowed}($v$,$G$)
\STATE {\bf Initiate:} $A=\emptyset$.
\STATE Set $v$ as the source variable in $G$.
\FOR{$\{u,w\}\in E(G)$}
\STATE {\bf if} $d_G(v,u)<d_G(v,w)$ {\bf then} $A=A\cup (u,w)$ {\bf end if}
\STATE {\bf if} $d_G(v,u)>d_G(v,w)$ {\bf then} $A=A\cup (w,u)$ {\bf end if} 
\ENDFOR
\STATE {\bf return} $A$.
\STATE {\bf end function}
\\ \vspace{-2mm} \hrulefill
\STATE {\bf function} \textsc{W}($v$,$G$)
\STATE $\overline{F}=$ Undirected version of elements of $\textsc{Flowed}(v,G)$.
\STATE $G'=G\backslash\overline{F}$.
\STATE Remove isolated vertices from $G'$.
\STATE {\bf return} $\prod_{\hat{G}\in\textit{comp}(G')}\sum_{u\in V(\hat{G})}\textsc{W}(u,\hat{G})$.
\STATE {\bf end function}
\\ \vspace{-2mm} \hrulefill
\STATE {\bf function} \textsc{RandEdge}($G$,$A$)
\STATE Set $v^*\in V(G)$ as the source variable of $G$ with 
\STATE ~~~~~~~probability $\textsc{W}(v^*,G)/\sum_{v\in V(G)}\textsc{W}(v,G)$.
\STATE $A=A\cup\textsc{Flowed}(v^*,G)$.
\STATE $\overline{F}=$ Undirected version of elements of $\textsc{Flowed}(v^*,G)$.
\STATE $G'=G\backslash\overline{F}$.
\STATE Remove isolated vertices from $G'$.
\STATE $ A=A\bigcup_{\hat{G}\in\textit{comp}(G')}\textsc{RandEdge}(\hat{G},A)$.
\STATE {\bf return} $A$.
\STATE {\bf end function}
 \caption*{{\bf Subroutine 1} Unbiased $\mathcal{D}(\mathcal{I})$ Estimator Subroutine}
\end{algorithmic}
\label{algorithm:UE}
\end{algorithm}

\begin{theorem}
\label{thm:uconv}
For an estimator with $\mathbb{E}[D(\mathcal{I},G'_i)]=\mathcal{D}(\mathcal{I})$, given set $\mathcal{I}$ and $\epsilon,\delta>0$, if $N=|\mathcal{G}'|>\frac{|E(Ess(G^*))|(2+\epsilon)}{\epsilon^2}\ln(\frac{2}{\delta})$, then
\vspace{-2mm}
\[
\mathcal{D}(\mathcal{I})(1-\epsilon)<\hat{\mathcal{D}}(\mathcal{I})<\mathcal{D}(\mathcal{I})(1+\epsilon),
\]
with probability larger than $1-\delta$.
\end{theorem}
\vspace{-2mm}
\begin{proof}
 Define $X_i=\frac{D(\mathcal{I},G'_i)}{|E(Ess(G^*))|}$, for $i\in\{1,...,N\}$. 
 By the assumption of the theorem, $\mathbb{E}[X_i]=\frac{1}{|E(Ess(G^*))|}\mathcal{D}(\mathcal{I})$ (Note that as stated before, this assumption is satisfied by Subroutine 1).
Using Chernoff bound we have
 \vspace{-2mm}
\begin{align*}
P(|\sum_{i=1}^NX_i-&\frac{N}{|E(Ess(G^*))|}\mathcal{D}(\mathcal{I})|\ge\epsilon\frac{N}{|E(Ess(G^*))|}\mathcal{D}(\mathcal{I}))\\
&\le 2 \exp (-\frac{N\epsilon^2}{|E(Ess(G^*))|(2+\epsilon)}\mathcal{D}(\mathcal{I}))\\
&\le 2 \exp (-\frac{N\epsilon^2}{|E(Ess(G^*))|(2+\epsilon)}).
\end{align*}
Therefore,
$P(|\frac{1}{N}\sum_{i=1}^ND(\mathcal{I},G'_i)-\mathcal{D}(S)|\ge\epsilon\mathcal{D}(\mathcal{I}))\le 2 \exp (-\frac{N\epsilon^2}{|E(Ess(G^*))|(2+\epsilon)})$.
Hence,
$P(|\hat{\mathcal{D}}(\mathcal{I})-\mathcal{D}(\mathcal{I})|<\epsilon\mathcal{D}(\mathcal{I}))
> 1-2 \exp (-\frac{N\epsilon^2}{|E(Ess(G^*))|(2+\epsilon)})$.\\
Setting $N>\frac{|E(Ess(G^*))|(2+\epsilon)}{\epsilon^2}\ln(\frac{2}{\delta})$, upper bounds the right hand side with $1-\delta$ and concludes the desired result.
\end{proof}
\vspace{-3mm}

For any $\epsilon'>0$, General Greedy Algorithm provides us with a $(1-\frac{1}{e}-\epsilon')$-approximation of the optimal value as formalized in the following theorem.

\begin{theorem}
\label{thm:app}
For any $\epsilon',\delta'>0$, there exists $\epsilon,\delta>0$, such that if for any set $\mathcal{I}$, $\mathcal{D}(\mathcal{I})(1-\epsilon)<\hat{\mathcal{D}}(\mathcal{I})<\mathcal{D}(\mathcal{I})(1+\epsilon)$ with probability larger than $1-\delta$, then Algorithm \ref{algorithm:GG} is a $(1-\frac{1}{e}-\epsilon')$-approximation algorithm with probability larger than $1-\delta'$.
\end{theorem}
\vspace{-2mm}
See the supplementary materials for the proof.

\begin{theorem}
\label{thm:comp}
The computational complexity of General Greedy algorithm is $O(kNn^{\Delta+1})$ where $n$ and $\Delta$ are the order and maximum degree of $Ess(G^*)$, respectively.
\end{theorem}
\vspace{-2mm}
See the supplementary materials for the proof.

\vspace{-2mm}
\subsection{A Fast $\mathcal{D}(\mathcal{I})$ Estimator}
\label{subsec:fast}

\begin{algorithm}[t]
\begin{algorithmic}
\STATE {\bf Input:} $Ess(G^*)$, intervention target set $\mathcal{I}$, and $N$.
\STATE {\bf initiate:}  $\mathcal{G'}=\emptyset$
\FOR{$i=1$ to $N$, generate $G'_i$ as follows:}
\STATE Uniformly shuffle the order of the elements of $V$.
\WHILE{the induced subgraph on any subset of size 3 of the variables is not directed, or a directed cycle, or a v-structure which was not in $Ess(G^*)$}
\FOR{all $\{v_i,v_j,v_k\}\subseteq V$}
\STATE Orient the undirected edges among $\{v_i,v_j,v_k\}$ independently according to $\textit{Bern}(\frac{1}{2})$ until it becomes a directed structure which is not a directed cycle or a v-structure which was not in $Ess(G^*)$.
\ENDFOR
\ENDWHILE
\STATE $\mathcal{G'}=\mathcal{G'}\uplus G'_i$
\ENDFOR
\STATE {\bf Output:} $\hat{\mathcal{D}}(\mathcal{I})=\frac{1}{|\mathcal{G}'|}\sum_{G'_i\in\mathcal{G}'}D(\mathcal{I},G'_i)$
 \caption*{{\bf Subroutine 2} Fast $\mathcal{D}(\mathcal{I})$ Estimator Subroutine}
\end{algorithmic}
\label{algorithm:FE}
\end{algorithm}
The computational complexity of Subroutine 1 is $O(Nn^{\Delta})$, which may be intractable when the upper bound on the degree of the input graph is large.
Therefore, we propose a fast and efficient estimator for $\mathcal{D}(\mathcal{I})$, better suited for graphs with large degree.
Although this estimator is not unbiased, our extensive experimental results confirm that the sampling distribution of the sampler used in this estimator is very close to uniform.

The pseudo-code of the proposed estimator is presented in Subroutine 2. In this subroutine for the given mixed graph $Ess(G^*)$, we generate $N$ DAGs  from the Markov equivalence class of $G^*$ as follows:
We consider all subsets of size 3 from $V$ in a uniformly random order (achieved by uniformly shuffling the labels of elements of $V$). For each subset $\{v_i, v_j,v_k\}$, we orient the undirected edges among $\{v_i, v_j, v_k\}$ independently according to Bernoulli$(\frac{1}{2})$ distribution. If the resulting orientation on the induced subgraph on $\{v_i, v_j, v_k\}$ is a directed cycle or a new v-structure, which was not in $Ess(G^*)$, we redo the orienting. We keep checking all the subsets of size 3 until the induced subgraph on all of them are directed and none of them is a new v-structure, which did not exist in $Ess(G^*)$, or a directed cycle. 
\begin{lemma}
\label{lem:inMEC}
Each generated directed graph $G'_i$ in Fast $\mathcal{D}(\mathcal{I})$ Estimator Subroutine belongs to the Markov equivalence class of $G^*$.
\end{lemma}
\vspace{-3mm}
See the supplementary materials for the proof.

We add the generated DAG to a multiset $\mathcal{G}'$.
Finally, we calculate the estimated value $\hat{\mathcal{D}}(\mathcal{I})$ on $\mathcal{G}'$ instead of $\mathcal{G}$ as $\hat{\mathcal{D}}(\mathcal{I})=\frac{1}{|\mathcal{G}'|}\sum_{G'_i\in\mathcal{G}'}D(\mathcal{I},G'_i)$.

\vspace{-2mm}
\section{Improved Greedy Algorithm}
\label{sec:algorithm}


 We exploit the submodularity of function $\mathcal{D}$ to implement an accelerated variant of the general greedy algorithm through {\it lazy} evaluations, originally proposed by Minoux\footnote{There are improved versions of this algorithm in the literature \cite{mirzasoleiman2015lazier}.} \cite{minoux1978accelerated}. 
In each round of the general greedy algorithm, we check the marginal gain $\Delta_v(\mathcal{I})$ for all remaining vertices in $V\backslash \mathcal{I}$. Note that as a consequence of submodularity of function $\mathcal{D}$, the set function $\Delta_v$ is monotonically decreasing.
The main idea of the improved greedy algorithm is to take advantage of this property to avoid checking all the variables in each round of the algorithm. More specifically, suppose for vertices $v_1$ and $v_2$, in the $i$-th round of the algorithm we have obtained marginal gains $\Delta_{v_1}(\mathcal{I}_i)>\Delta_{v_2}(\mathcal{I}_i)$. If in the $(i+1)$-th round, we calculate $\Delta_{v_1}(\mathcal{I}_{i+1})$ and observe that $\Delta_{v_1}(\mathcal{I}_{i+1})>\Delta_{v_2}(\mathcal{I}_i)$, from monotonic decreasing property of function $\Delta_v$, we can conclude that $\Delta_{v_1}(\mathcal{I}_{i+1})>\Delta_{v_2}(\mathcal{I}_{i+1})$, and hence, there is no need to calculate $\Delta_{v_2}(\mathcal{I}_{i+1})$.

Improved Greedy Algorithm is presented in Algorithm \ref{algorithm:IG}. The idea can be formalized as follows: We define a profit parameter $p_v$ for each variable $v$ and initialize the value for all variables with $\infty$. Moreover, we define an update flag $\textit{upd}_v$ for all variables, which will be set to false at the beginning of every round of the algorithm and will be switched to true if we update $p_v$ with the value of the marginal gain of vertex $v$.
In each round, the algorithm picks vertex $v\in V\backslash \mathcal{I}$ with the largest profit, updates its profit with the value of the marginal gain of $v$, and sets $\textit{upd}_v$ to true. This process is repeated until the vertex with the largest profit is already updated, i.e., its update flag is true. Then we add this vertex to $\mathcal{I}$ and end the round.
For example, if in a round, the vertex $v$ has the highest profit and after updating the profit of this vertex, $p_v$ is still larger than all the other profits, we do not need to evaluate the marginal gain of any other vertex and we add $v$ to $\mathcal{I}$.

The correctness of the Improved Greedy Algorithm follows directly from submodularity of function $\mathcal{D}$. Theorem \ref{thm:app} holds for Algorithm \ref{algorithm:IG} as well, that is, for any $\epsilon'>0$, Improved Greedy Algorithm provides us with a $(1-\frac{1}{e}-\epsilon')$-approximation of the optimal value.
This algorithm can lead to orders of magnitude performance speedup, as shown in \cite{leskovec2007cost}.

\vspace{-2mm}
\section{Experimental Results}
\label{sec:experiments}
\subsection{Synthetic Graphs}
In this subsection, we evaluate the performance of Ran-GrID approach on synthetically generated chordal graphs\footnote{A chord of a cycle is an edge not in the cycle whose endpoints are in the cycle. A hole in a graph is a cycle of length at least 4 having no chord. A graph is chordal if it has no hole.}. Subroutine 1 and Algorithm \ref{algorithm:IG} are used in our experiments. We use randomly chosen perfect elimination ordering (PEO)\footnote{ A perfect elimination ordering $\{v_1 , v_2,...,  v_n \}$ on the vertices of an undirected chordal graph is such that for all $i$, the induced neighborhood of $v_i$ on the subgraph formed by $\{v_1,v_2,...,v_{i-1}\}$ is a clique.} of the vertices to generate our underlying chordal graphs \cite{hauser2014two,shanmugam2015learning}. For each graph, we pick a random ordering of the vertices. Starting from the vertex $v$ with the highest order, we connect all the vertices with lower order to $v$ with probability inversely proportional to the order of $v$. Then, we connect all the parents of $v$ with directed edges, where each directed edge is oriented from the parent with the lower order to the parent with the higher order. In order to make sure that the generated graph will be connected, if vertex $v$ is not connected to any of the vertices with the lower order, we pick one of them uniformly at random and set it as the parent of $v$.



\begin{algorithm}[t]
\begin{algorithmic}
 \STATE {\bf Input:} Joint distribution over $V$, and budget $k$.
\STATE Obtain $Ess(G^*)$ by performing a CCI algorithm.
\STATE {\bf initiate:} $\mathcal{I}_0=\emptyset$, and $p_v=\infty$, $\forall v\in V$.
\FOR{$i=1$ to $k$}
\STATE $upd_v=$\texttt{false}, $\forall v \in V\backslash \mathcal{I}_{i-1}$
\WHILE{\texttt{true}}
\STATE $v^* =\arg\max_{v\in V\backslash \mathcal{I}_{i-1}} p_v$
\IF{$upd_{v^*}$}
\STATE $\mathcal{I}_i=\mathcal{I}_{i-1}\cup \{v^*\}$
\STATE {\bf break};
\ELSE
\STATE $p_{v^*}= \hat{\mathcal{D}}(\mathcal{I}_{i-1}\cup \{v^*\})-\hat{\mathcal{D}}(\mathcal{I}_{i-1})$
\STATE $upd_{v^*}=\texttt{true}$
\ENDIF
\ENDWHILE
\ENDFOR
\STATE {\bf Output:} $\hat{\mathcal{I}}=\mathcal{I}_k$
 \caption{Improved Greedy Algorithm}
 \label{algorithm:IG}
\end{algorithmic}
\end{algorithm}

\begin{figure*}[t]
\centering
\begin{minipage}{.333\textwidth}
  \centering
  \includegraphics[width=.98\linewidth]{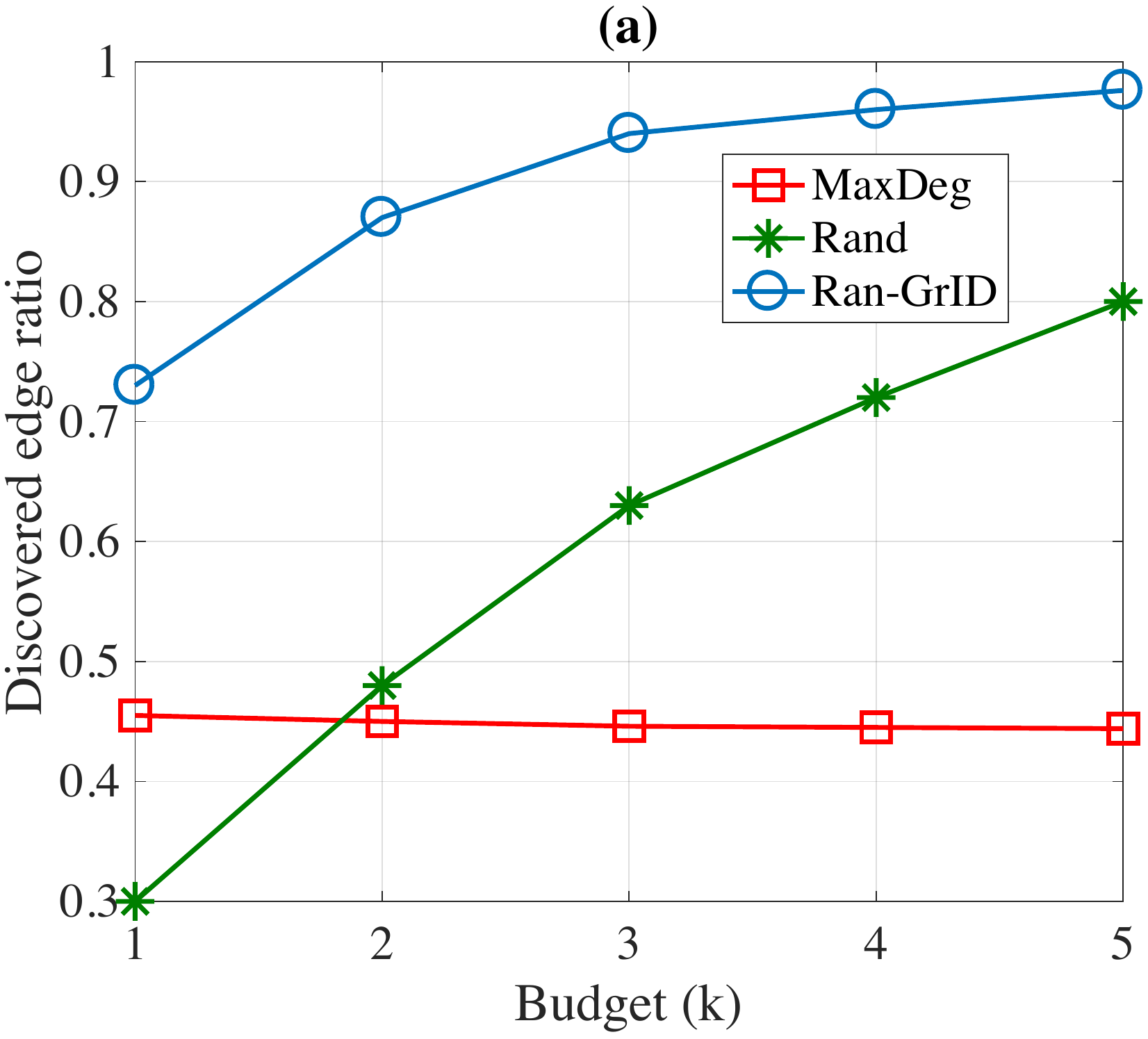}
\end{minipage}%
\begin{minipage}{.333\textwidth}
  \centering
  \includegraphics[width=.98\linewidth]{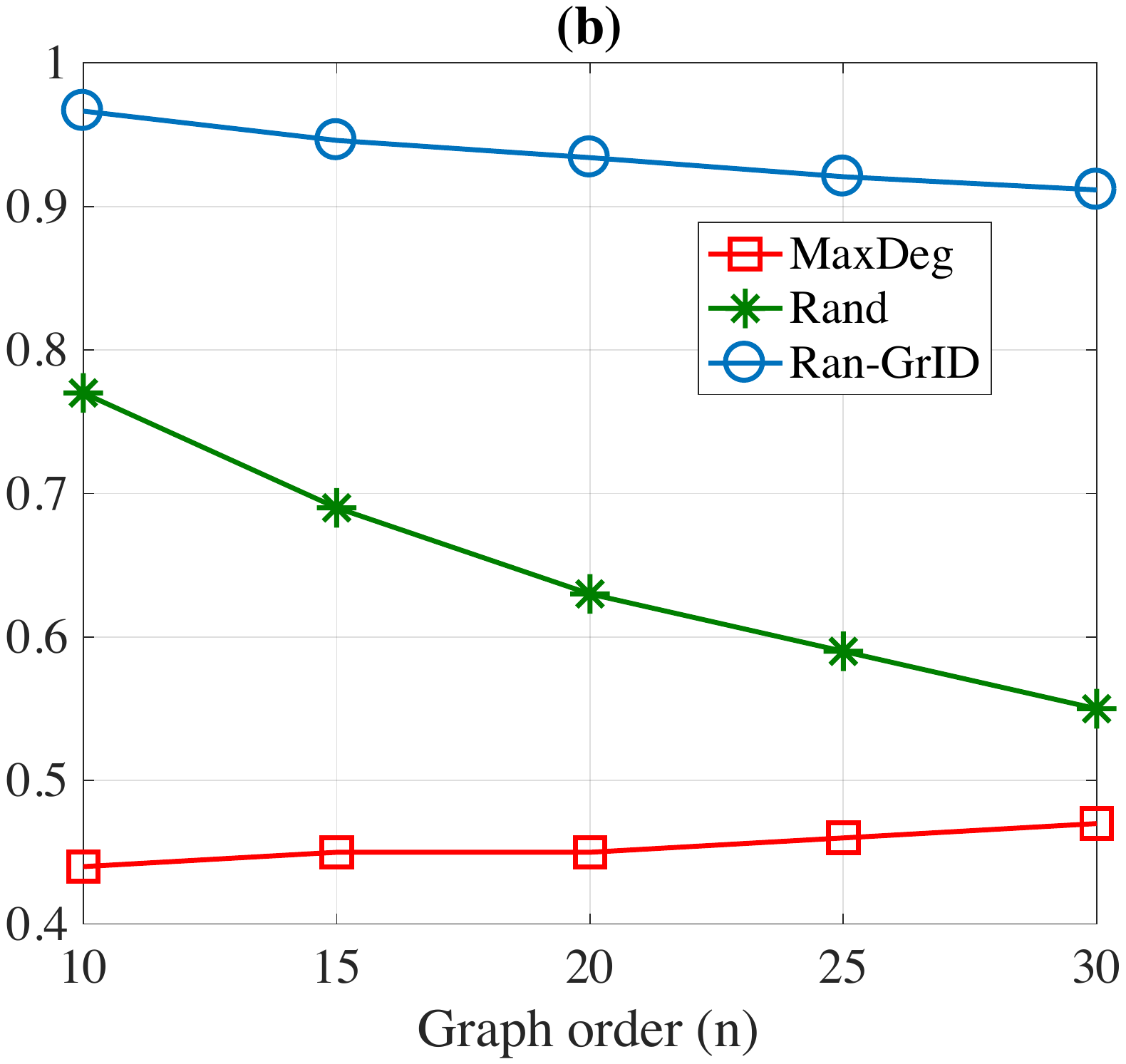}
\end{minipage}%
\begin{minipage}{.333\textwidth}
  \centering
  \includegraphics[width=.98\linewidth]{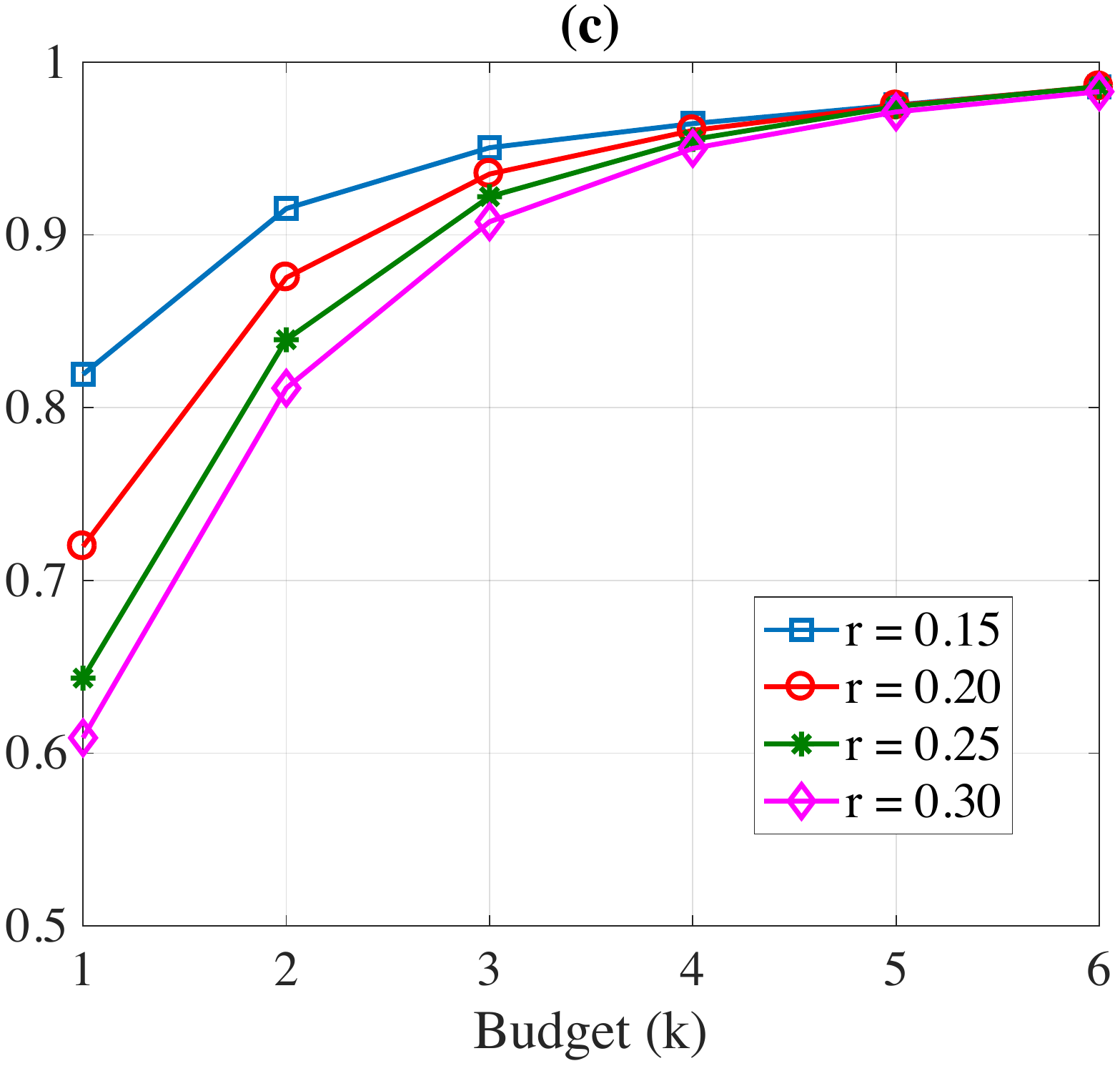}
\end{minipage}
\caption{Discovered edge ratio versus (a) budget for $n=20$, (b) graph orders for $k=3$, (c) budget for $n=20$ and different densities.} 
\vspace{-3mm}
\label{fig:exps}
\end{figure*}

%
%
%

To evaluate the performance of the proposed algorithm, for any underlying graph, we consider the ratio of the number of edges whose directions are discovered merely as a result of interventions to the number of edges whose directions were not resolved from the observational data. Note that due to our specific graph generating approach, the orientation of none of the edges is learned from the observational data. We compared Ran-GrID with two naive approaches: 1. Rand: Selecting intervened variables randomly, 2. MaxDeg: Sorting the list of variables based on the number of undirected edges connected to them in descending order and picking the first $k$ variables from the sorted list.

We generated 100 instances of chordal DAGs of order 20. 
Figure \ref{fig:exps}(a) depicts the discovered edge ratio with respect to the budget $k$. As seen in this figure, three interventions suffices to discover the direction of more than 90\% of the edges whose direction was unknown prior to performing interventions. 
Further, to investigate the effect of the order of the graph on the performance of the proposed algorithm and two naive approaches, we evaluated the discovered edge ratio for budget $k=3$ on graphs with order $n\in \{10,15,20,25,30\}$ in Figure \ref{fig:exps}(b). As it can be seen in the figure, the discovered edge ratio for the proposed approach is greater than $91\%$ for all $n\leq 30$. The performance of Rand approach degrades dramatically as $n$ increases. Moreover, MaxDeg approach has even lower performance than Rand approach. We also studied the effect of graph density on the performance of proposed algorithm. Let $r$ be the ratio of average number of edges to $\binom{n}{2}$.
The discovered edge ratio for chordal DAGs of order 20 versus budget for different densities is depicted in  Figure \ref{fig:exps}(c).

Furthermore, to compare the performance of the proposed algorithm with the optimal solution, we generated 100 instances of chordal DAGs of order 10 and performed a brute force search to find the optimal solution of \eqref{eq:optproblem} for budget $k=2$. The discovered edge ratio was $0.9$ and $0.916$ for our proposed algorithm and the optimal solution, respectively. For the aforementioned setting, the running time of the proposed approach on a machine with Intel Core i7 processor and 16 GB of RAM was $216$ seconds while the one of the brute force approach was greater than $6000$ seconds.

\vspace{-2mm}
\subsection{Real Graphs}
We evaluated the performance of the proposed Improved Greedy Algorithm in gene regulatory networks (GRN). GRN is a collection of biological regulators that interact with each other. In GRN, the transcription factors are the main players to activate genes. The interactions between transcription factors and regulated genes in a species genome can be presented by a directed graph. In this graph, links are drawn whenever a transcription factor regulates a gene's expression. Moreover, some of vertices have both functions, i.e., are both transcription factor and regulated gene. 

We considered GRNs in ``DREAM 3 In Silico Network" challenge, conducted in 2008 \cite{marbach2009generating}. The networks in this challenge were extracted from known biological interaction networks. Since we know the true causal structures in these GRNs, we can obtain $Ess(G^*)$ and give it as an input to the proposed algorithm. Figure \ref{fig:sim3} depicts the discovered edge ratio in five networks extracted from GRNs of E-coli and Yeast bacteria with budget $k=5$. The order of each network is 100. As it can be seen, the discovered edge ratio is at least $0.65$ in all GRNs.
\vspace{-2mm}
\section{Conclusion}
We studied the problem of experiment design for causal structure learning when only a limited number of experiments are available. In our model, each experiment consists of intervening on a single vertex. 
Also, experiments are designed merely based on the result of an initial purely observational test, which enables the experimenter to perform the interventional tests in parallel. 
We addressed the following question: ``How much of the causal structure can be learned when only a limited number of experiments are available?'' 
We formulated the problem of finding the best intervention target set as an optimization problem which aims to maximize the average number of edges whose directions are discovered. 
We introduce, for the first time, the use of submoular optimization in the context of causal experimental design by showing that
the objective function is monotonically increasing and submodular. Consequently, the greedy algorithm is a $(1-\frac{1}{e})$-approximation algorithm for this problem. Moreover, we proposed estimation methods in order to compute the objective function for a given set of intervention targets. 
We further presented an accelerated variant of the greedy algorithm, which can achieve orders of magnitude performance speedup. 
We evaluated our proposed improved greedy algorithm on synthetic as well as real graphs. The results showed that a significant portion of the causal systems can be learned by only a few number of interventions.

\begin{figure}[t]
\begin{center}
\centerline{\includegraphics[scale=0.4]{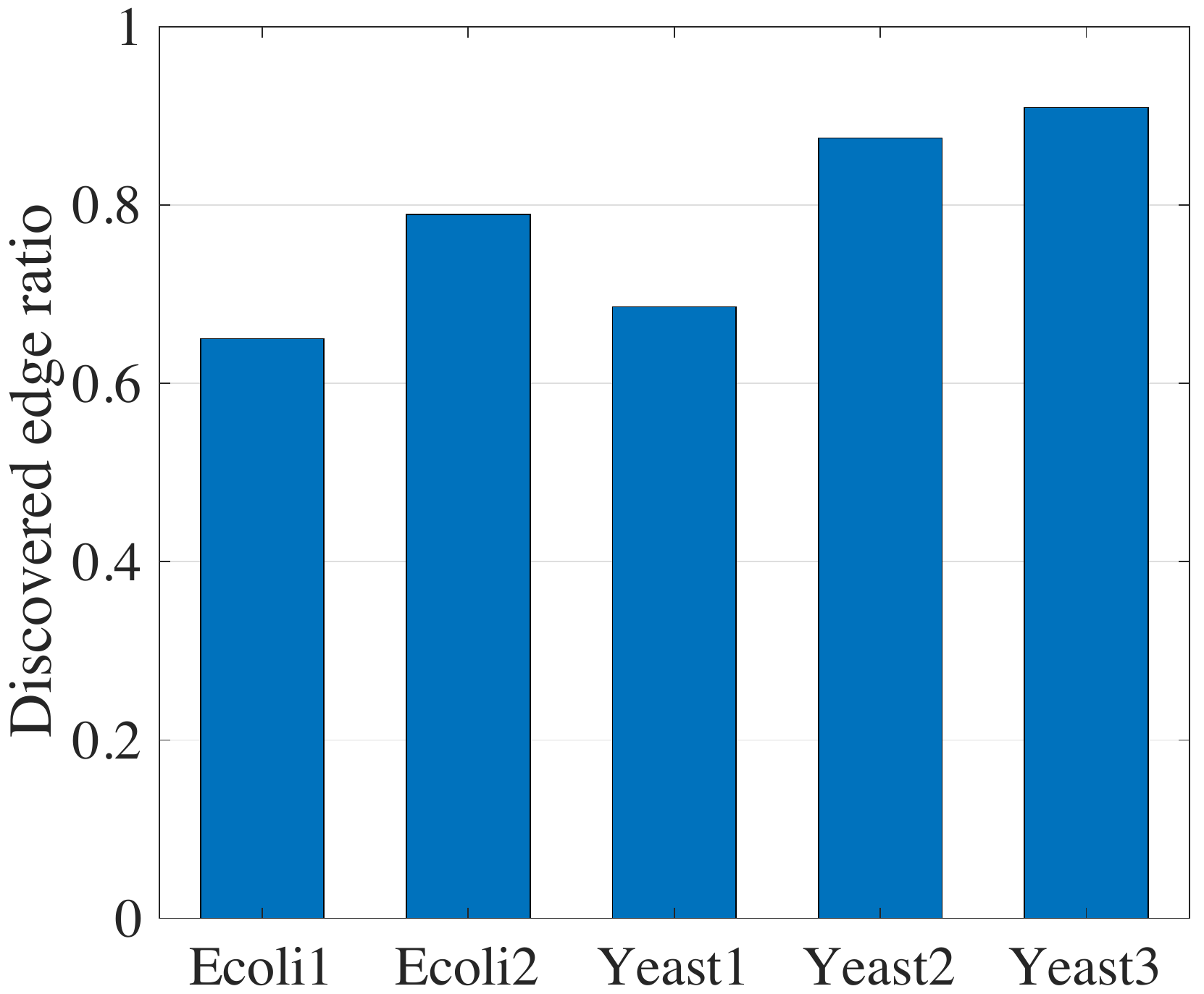}}
\caption{Discovered edge ratio in five GRNs from DREAM 3 challenge.}
\vspace{-8mm}
\label{fig:sim3}
\end{center}
\end{figure}

\newpage

\section*{Acknowledgements}
Ghassami, Salehkaleybar, and Kiyavash’s work was in part supported by Navy grant N00014-16-1-2804, and Army grant W911NF-15-1-0281. Bareinboim's work was in part supported by grants from NSF IIS-1704352 and IIS-1750807 (CAREER).


\nocite{langley00}

\bibliography{Bibliography-File}
\bibliographystyle{icml2018}

\newpage~
\newpage

\begin{appendices}

\section{Example of Comparison with the Influence Maximization Problem}

Suppose $k=1$. Figure \ref{fig:ex} depicts a graph for which the optimal solution to the influence maximization problem is different from the optimal solution to the budgeted experiment design problems. Clearly, influencing vertex $v_1$ leads to influencing all the vertices in the graph, and hence, this vertex is the solution to the influence maximization problem. But, intervening on $v_1$ leads to discovering the orientation of only 3 edges, while intervening on, say $v_2$, leads to discovering the orientation of 5 edges.

\section{Proof of Lemma \ref{lem:mono}}

First we show that for a given directed graph $G_i\in\mathcal{G}$ the function $D(\mathcal{I},G_i)$ is a monotonically increasing function of $\mathcal{I}$. 
In the proposed method, intervening on elements of $\mathcal{I}$, we first discover the orientation of the edges in $A^{(\mathcal{I})}_{G_i}$, and then applying the Meek rules, we possibly learn the orientation of some extra edges. 
Having $\mathcal{I}_1\subseteq \mathcal{I}_2$ implies that $A^{(\mathcal{I}_1)}_{G_i}\subseteq A^{(\mathcal{I}_2)}_{G_i}$. Therefore using $\mathcal{I}_2$, we have more information about the direction of edges. Hence, in the step of applying Meek rules, by soundness and order-independence of Meek algorithm, we recover the direction of more extra edges, i.e., $R(A^{(\mathcal{I}_1)}_{G_i},G_i)\subseteq R(A^{(\mathcal{I}_2)}_{G_i},G_i)$, which in turn implies that $D(\mathcal{I}_1,G_i)\le D(\mathcal{I}_2,G_i)$.
Finally, from the relation $\mathcal{D}(\mathcal{I})=\frac{1}{|\mathcal{G}|}\sum_{G_i\in\mathcal{G}}D(\mathcal{I},G_i)$, the desired result is immediate.

\section{Proof of Lemma \ref{lem:nofusion}}
The direction $R(A_{G^*}^{(\mathcal{I}_1)},G^*)\cup R(A_{G^*}^{(\mathcal{I}_2)},G^*)\subseteq R(A_{G^*}^{(\mathcal{I}_1\cup\mathcal{I}_2)},G^*)$ is proved in the proof of Lemma \ref{lem:mono}. 
Also, as observed in the proof of Lemma \ref{lem:mono}, we have $R(A_{G^*}^{(\mathcal{I}_1\cup\mathcal{I}_2)},G^*)\subseteq R(R(A_{G^*}^{(\mathcal{I}_1)},G^*)\cup R(A_{G^*}^{(\mathcal{I}_2)},G^*),G^*)$. Therefore, in order to prove that $R(A_{G^*}^{(\mathcal{I}_1\cup\mathcal{I}_2)},G^*)\subseteq R(A_{G^*}^{(\mathcal{I}_1)},G^*)\cup R(A_{G^*}^{(\mathcal{I}_2)},G^*)$, it suffices to show that $R(R(A_{G^*}^{(\mathcal{I}_1)},G^*)\cup R(A_{G^*}^{(\mathcal{I}_2)},G^*),G^*)\subseteq R(A_{G^*}^{(\mathcal{I}_1)},G^*)\cup R(A_{G^*}^{(\mathcal{I}_2)},G^*)$, for which it suffices to show that if $e\not\in R(A_{G^*}^{(\mathcal{I}_1)},G^*)$ and $e\not\in R(A_{G^*}^{(\mathcal{I}_2)},G^*)$, then $e\not\in R(R(A_{G^*}^{(\mathcal{I}_1)},G^*)\cup R(A_{G^*}^{(\mathcal{I}_2)},G^*),G^*)$.\\

\textit{Proof by contradiction.} Let $e\not\in R(A_{G^*}^{(\mathcal{I}_1)},G^*)$ and $e\not\in R(A_{G^*}^{(\mathcal{I}_2)},G^*)$, but its orientation is learned in the first iteration of applying Meek rules to $R(A_{G^*}^{(\mathcal{I}_1)},G^*)\cup R(A_{G^*}^{(\mathcal{I}_2)},G^*)\cup A(\textit{Ess}(G^*))$. Then, we have learned the orientation of $e$ due to one of Meek rules \cite{verma1992algorithm}:
\begin{itemize}
\item \textit{Rule 1.} $e=\{a,b\}$ is oriented as $(a,b)$ if $\exists c$ s.t. $e_1=(c,a)\in R(A_{G^*}^{(\mathcal{I}_1)},G^*)\cup R(A_{G^*}^{(\mathcal{I}_2)},G^*)\cup A(\textit{Ess}(G^*))$, and $\{c,b\}\not\in$ skeleton of $G^*$.
\item \textit{Rule 2.} $e=\{a,b\}$ is oriented as $(a,b)$ if $\exists c$ s.t. $e_1=(a,c)\in R(A_{G^*}^{(\mathcal{I}_1)},G^*)\cup R(A_{G^*}^{(\mathcal{I}_2)},G^*)\cup A(\textit{Ess}(G^*))$, and $e_2=(c,b)\in R(A_{G^*}^{(\mathcal{I}_1)},G^*)\cup R(A_{G^*}^{(\mathcal{I}_2)},G^*)\cup A(\textit{Ess}(G^*))$.
\item \textit{Rule 3.} $e=\{a,b\}$ is oriented as $(a,b)$ if $\exists c,d$ s.t. $e_1=(c,b)\in R(A_{G^*}^{(\mathcal{I}_1)},G^*)\cup R(A_{G^*}^{(\mathcal{I}_2)},G^*)\cup A(\textit{Ess}(G^*))$, $e_2=(d,b)\in R(A_{G^*}^{(\mathcal{I}_1)},G^*)\cup R(A_{G^*}^{(\mathcal{I}_2)},G^*)\cup A(\textit{Ess}(G^*))$, $\{a,c\}\in$ skeleton of $G^*$, $\{a,d\}\in$ skeleton of $G^*$, and $\{c,d\}\not\in$ skeleton of $G^*$.
\item \textit{Rule 4.} $e=\{a,b\}$ is oriented as $(a,b)$ and $e=\{b,c\}$ is oriented as $(c,b)$ if $\exists d$ s.t. $e_1=(d,c)\in R(A_{G^*}^{(\mathcal{I}_1)},G^*)\cup R(A_{G^*}^{(\mathcal{I}_2)},G^*)\cup A(\textit{Ess}(G^*))$, $\{a,c\}\in$ skeleton of $G^*$, $\{a,d\}\in$ skeleton of $G^*$, and $\{b,d\}\not\in$ skeleton of $G^*$.
\end{itemize}




\begin{figure}[t]
\begin{center}
\centerline{\includegraphics[scale=0.4]{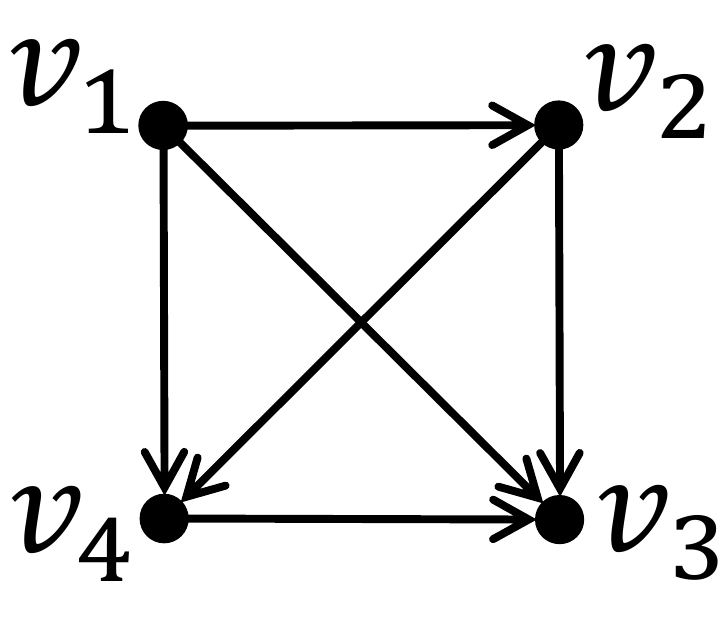}}
\caption{Example of comparison with the influence maximization problem.}
\label{fig:ex}
\end{center}
\end{figure}

In what follows, we show that the orientation of $e$ cannot be learned due to any of the Meek rules unless it belongs to $R(A_{G^*}^{(\mathcal{I}_1)},G^*)$ or $R(A_{G^*}^{(\mathcal{I}_2)},G^*)$.\\

\noindent
\textit{Rule 1.}

Without loss of generality, assume $e_1\in R(A_{G^*}^{(\mathcal{I}_1)},G^*)\cup A(\textit{Ess}(G^*))$. Therefore, we should have the condition of rule 1 satisfied when only intervening on $\mathcal{I}_1$ as well, which implies that $e\in R(A_{G^*}^{(\mathcal{I}_1)},G^*)$, which is a contradiction.\\


\noindent
\textit{Rule 2.}

If both $e_1$ and $e_2$ belong to $R(A_{G^*}^{(\mathcal{I}_1)},G^*)\cup A(\textit{Ess}(G^*))$ (or $R(A_{G^*}^{(\mathcal{I}_2)},G^*)\cup A(\textit{Ess}(G^*))$), then we should have the condition of rule 2 satisfied when only intervening on $\mathcal{I}_1$ (or $\mathcal{I}_2$) as well, which implies that $e\in R(A_{G^*}^{(\mathcal{I}_1)},G^*)$ (or $e\in R(A_{G^*}^{(\mathcal{I}_2)},G^*)$), which is a contradiction. Therefore, it suffices to show that the case that $e_1$ belongs to exactly one of $R(A_{G^*}^{(\mathcal{I}_1)},G^*)\cup A(\textit{Ess}(G^*))$ or $R(A_{G^*}^{(\mathcal{I}_2)},G^*)\cup A(\textit{Ess}(G^*))$ and $e_2$ belongs only to the other one, does not happen. To this end, it suffices to show that there does not exist intervention target set $\mathcal{I}$ such that $e_1\in R(A_{G^*}^{(\mathcal{I})},G^*)\cup A(\textit{Ess}(G^*))$, and $e,e_2\not\in R(A_{G^*}^{(\mathcal{I})},G^*)\cup A(\textit{Ess}(G^*))$, i.e., there does not exist intervention target set $\mathcal{I}$ that has structure $S_0$, depicted in Figure \ref{fig:s0}, as a subgraph of $\textit{Ess}(G^*)$ after applying the orientations learned from $R(A_{G^*}^{(\mathcal{I})},G^*)$.

\begin{figure}[h]
\begin{center}
\centerline{\includegraphics[scale=0.2]{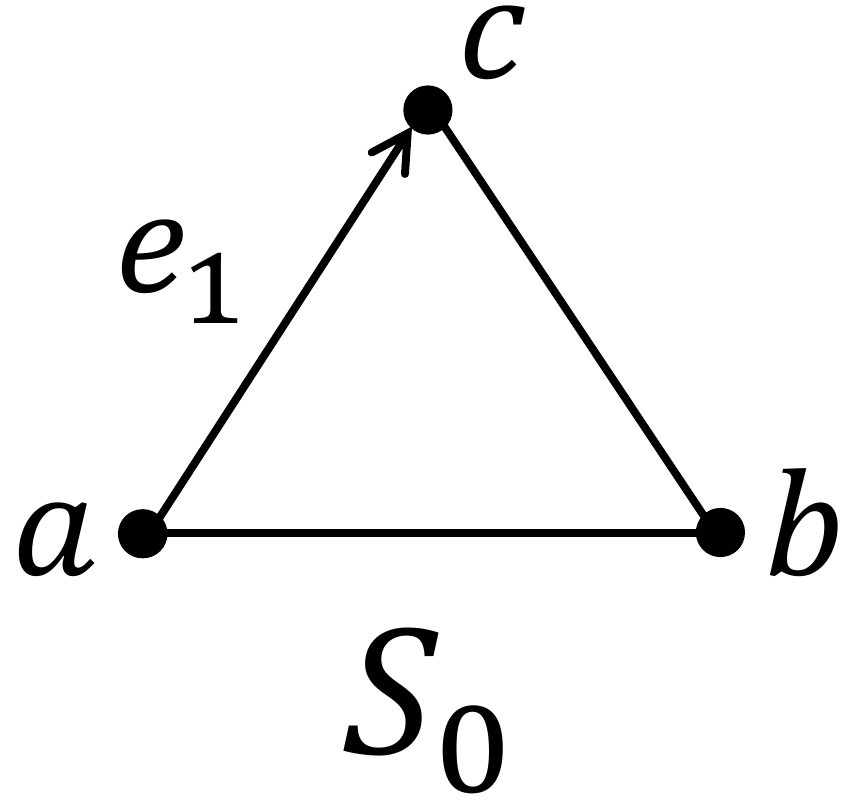}}
\caption{Structure $S_0$}
\vspace{-7mm}
\label{fig:s0}
\end{center}
\end{figure}

If $e_1\in A_{G^*}^{(\mathcal{I})}$, then $a\in \mathcal{I}$ or $c\in \mathcal{I}$, which implies $e\in A_{G^*}^{(\mathcal{I})}$ or $e_2\in A_{G^*}^{(\mathcal{I})}$, respectively, and hence, $e\in R(A_{G^*}^{(\mathcal{I})},G^*)$ or $e_2\in R(A_{G^*}^{(\mathcal{I})},G^*)$, respectively. Therefore, in either case, $e\in R(A_{G^*}^{(\mathcal{I})},G^*)$, and $S_0$ will not be a subgraph.
Therefore, $e_1\not\in A_{G^*}^{(\mathcal{I})}$, and hence, $e_1$ was learned by applying one of the Meek rules. We consider each or the rules in the following:
\begin{itemize}
\item If we have learned the orientation of $e_1$ from rule 1, then we should have had one of the structures in Figure \ref{fig:rule1} as a subgraph of $\textit{Ess}(G^*)$ after applying the orientations learned from $R(A_{G^*}^{(\mathcal{I})},G^*)$. In case of structure $S_1$, using rule 1 on subgraph induced on vertices $\{v_1,a,b\}$, we will also learn $(a,b)$. In case of structure $S_2$, using rule 4, we will also learn $(b,c)$. Therefore, we cannot learn only the direction of $e_1$ and hence, $S_0$ will not be a subgraph.
\begin{figure}[h]
\begin{center}
\centerline{\includegraphics[scale=0.2]{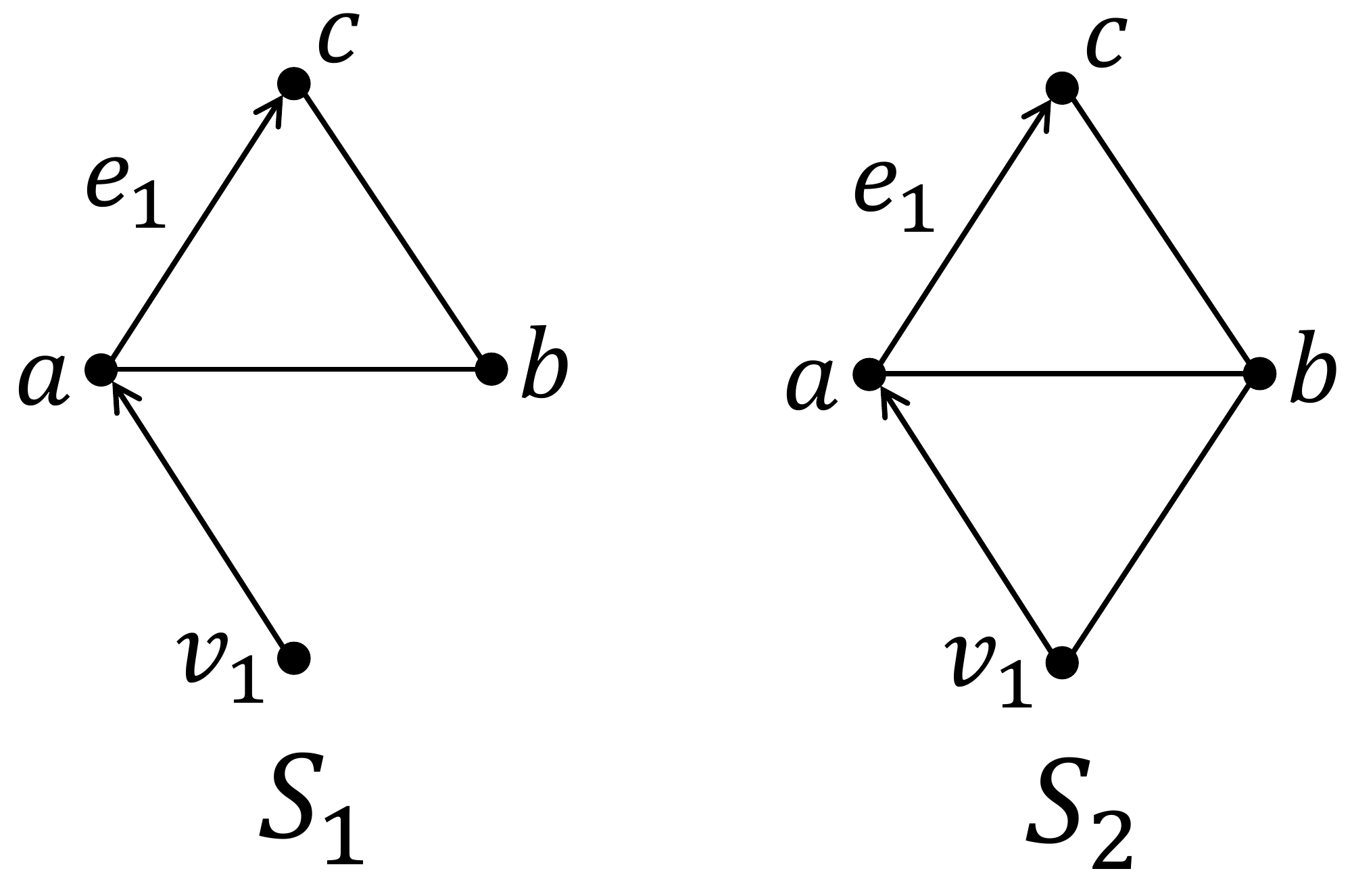}}
\caption{Rule 1}
\vspace{-7mm}
\label{fig:rule1}
\end{center}
\end{figure}

\item If we have learned the orientation of $e_1$ from rule 3, then we have had one of the structures in Figure \ref{fig:rule3} as a subgraph of $\textit{Ess}(G^*)$ after applying the orientations learned from $R(A_{G^*}^{(\mathcal{I})},G^*)$. In case of structures $S_3$ and $S_4$, using rule 1 on subgraph induced on vertices $\{v_2,c,b\}$, we will also learn $(c,b)$. In case of structure $S_5$, using rule 3 on subgraph induced on vertices $\{b,v_2,c,v_1\}$, we will also learn $(b,c)$. Therefore, we cannot learn only the direction of $e_1$ and hence, $S_0$ will not be a subgraph.
\begin{figure}[h]
\begin{center}
\centerline{\includegraphics[scale=0.2]{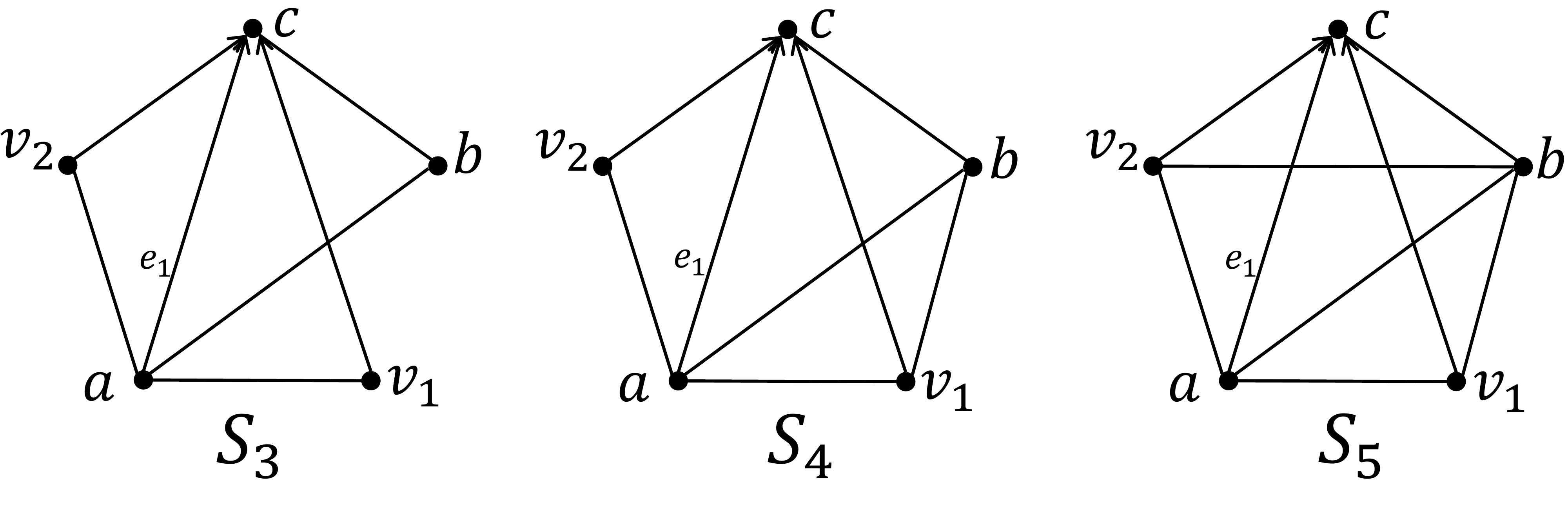}}
\caption{Rule 3}
\vspace{-7mm}
\label{fig:rule3}
\end{center}
\end{figure}

\item If we have learned the orientation of $e_1$ from rule 4, then we have had one of the structures in Figure \ref{fig:rule4} as a subgraph of $\textit{Ess}(G^*)$ after applying the orientations learned from $R(A_{G^*}^{(\mathcal{I})},G^*)$. In case of structures $S_6$, using rule 1 on subgraph induced on vertices $\{v_1,c,b\}$, we will also learn $(c,b)$. In case of structure $S_7$, using rule 1 on subgraph induced on vertices $\{v_2,v_1,b\}$, we will also learn $(v_1,b)$, and then using rule 4 on subgraph induced on vertices $\{b,a,v_2,v_1\}$, we will also learn $(a,b)$. In case of structures $S_{8}$, using rule 4 on subgraph induced on vertices $\{b,v_2,v_1,c\}$, we will also learn $(b,c)$. Therefore, we cannot learn only the direction of $e_1$ and hence, $S_0$ will not be a subgraph.
\begin{figure}[h]
\begin{center}
\centerline{\includegraphics[scale=0.2]{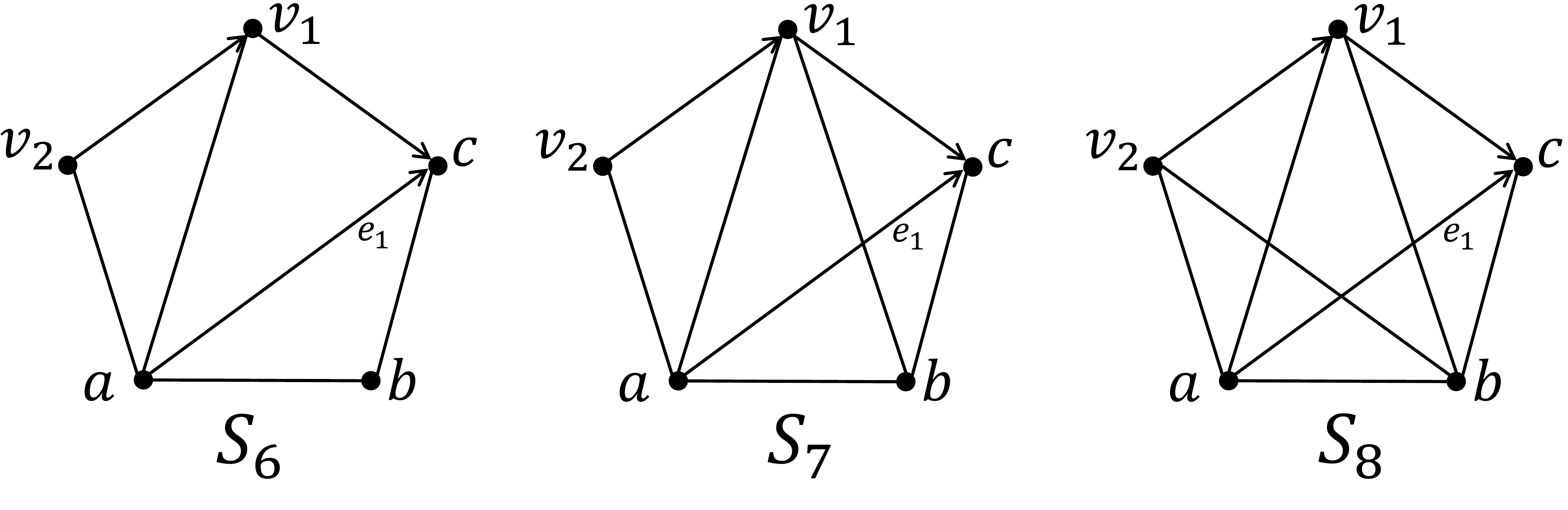}}
\caption{Rule 4}
\vspace{-7mm}
\label{fig:rule4}
\end{center}
\end{figure}

\item If we have learned the orientation of $e_1$ from rule 2, then we should have had one of the structures in Figure \ref{fig:rule2} as a subgraph of $\textit{Ess}(G^*)$ after applying the orientations learned from $R(A_{G^*}^{(\mathcal{I})},G^*)$. In case of structure $S_9$, using rule 1 on subgraph induced on vertices $\{v_1,c,b\}$, we will also learn $(c,b)$ and hence, $S_0$ will not be a subgraph.
In case of structure $S_{10}$, if $v_1\in \mathcal{I}$, then the direction of the edge $\{v_1,b\}$ will be also known. If the direction of this edge is $(v_1,b)$, then  using rule 2 on subgraph induced on vertices $\{a,v_1,b\}$, we will also learn $(a,b)$; otherwise, using rule 2 on subgraph induced on vertices $\{b,v_1,c\}$, we will also learn $(c,b)$. Therefore, $v_1\not\in \mathcal{I}$. Also, as mentioned earlier, $a\not\in \mathcal{I}$. Therefore, we have learned the orientation of $(a,v_1)$ from applying Meek rules.

In the triangle induced on vertices $\{v_1,b,a\}$, we have learned only the orientation of one edge, which is $(a,v_1)$. But as seen in structures $S_1$ to $S_9$, all of them lead to learning the orientation of at least 2 edges of a triangle. In the following, we will show that a structure of form $S_{10}$, does not lead to learning the orientation of only $(a,v_1)$ and making $S_{10}$ a subgraph either.



\begin{figure}[h]
\begin{center}
\centerline{\includegraphics[scale=0.2]{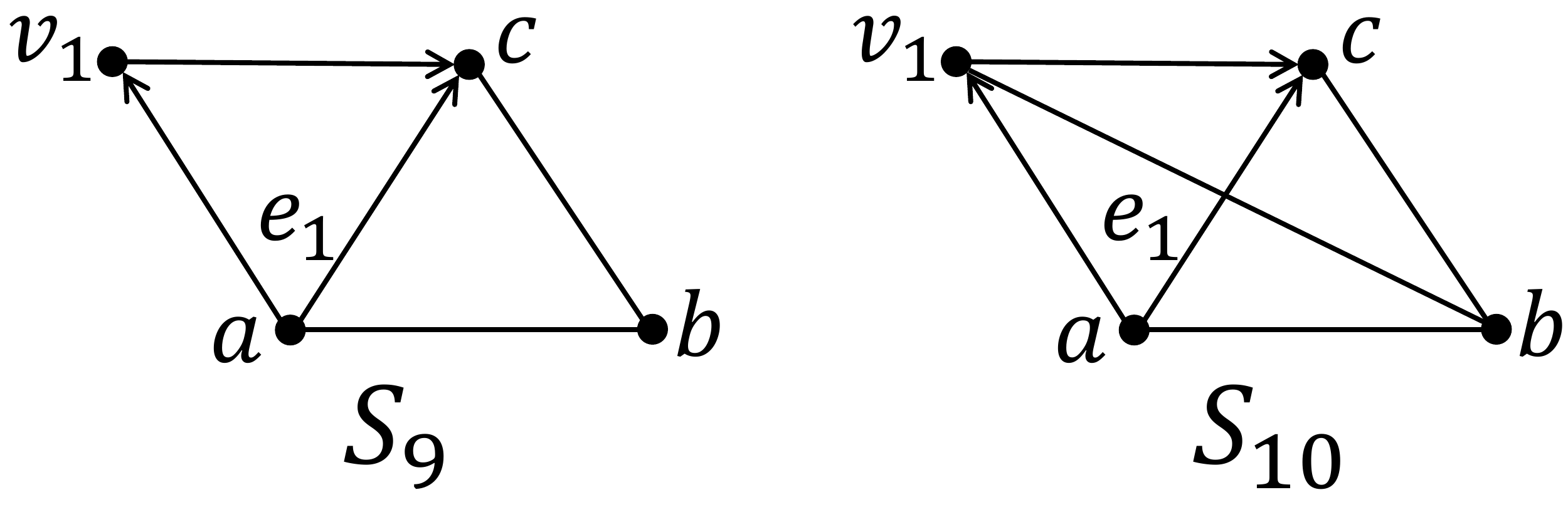}}
\caption{Rule 2}
\vspace{-7mm}
\label{fig:rule2}
\end{center}
\end{figure}

Suppose we had learned $(a,v_1)$ via a structure of form $S_{10}$, as depicted in Figure \ref{fig:s10}(a).
Using rule 4 on subgraph induced on vertices $\{v_2,v_1,c,b\}$, we will also learn $(b,c)$. Therefore, we should have the edge $\{v_2,c\}$ too. Also, using rule 2 on triangle induced on vertices $\{v_2,v_1,c\}$, the orientation of this edges should be $(v_2,c)$.
Therefore, in order to have $S_{10}$ as a subgraph, we need to have the structure depicted in Figure \ref{fig:s10}(b) as a subgraph.
As seen in Figure \ref{fig:s10}(b), we again have a structure similar to $S_{10}$: a complete skeleton $K_5$, which contains $(v_j,c)$, $(a,v_j)$, $\{v_j,b\}$, for $j\in\{1,2\}$ and $(v_2,v_1)$, with a triangle on vertices $\{v_2,b,a\}$, in which we have learned only the orientation of $(a,v_2)$.

\begin{figure}[h]
\begin{center}
\centerline{\includegraphics[scale=0.2]{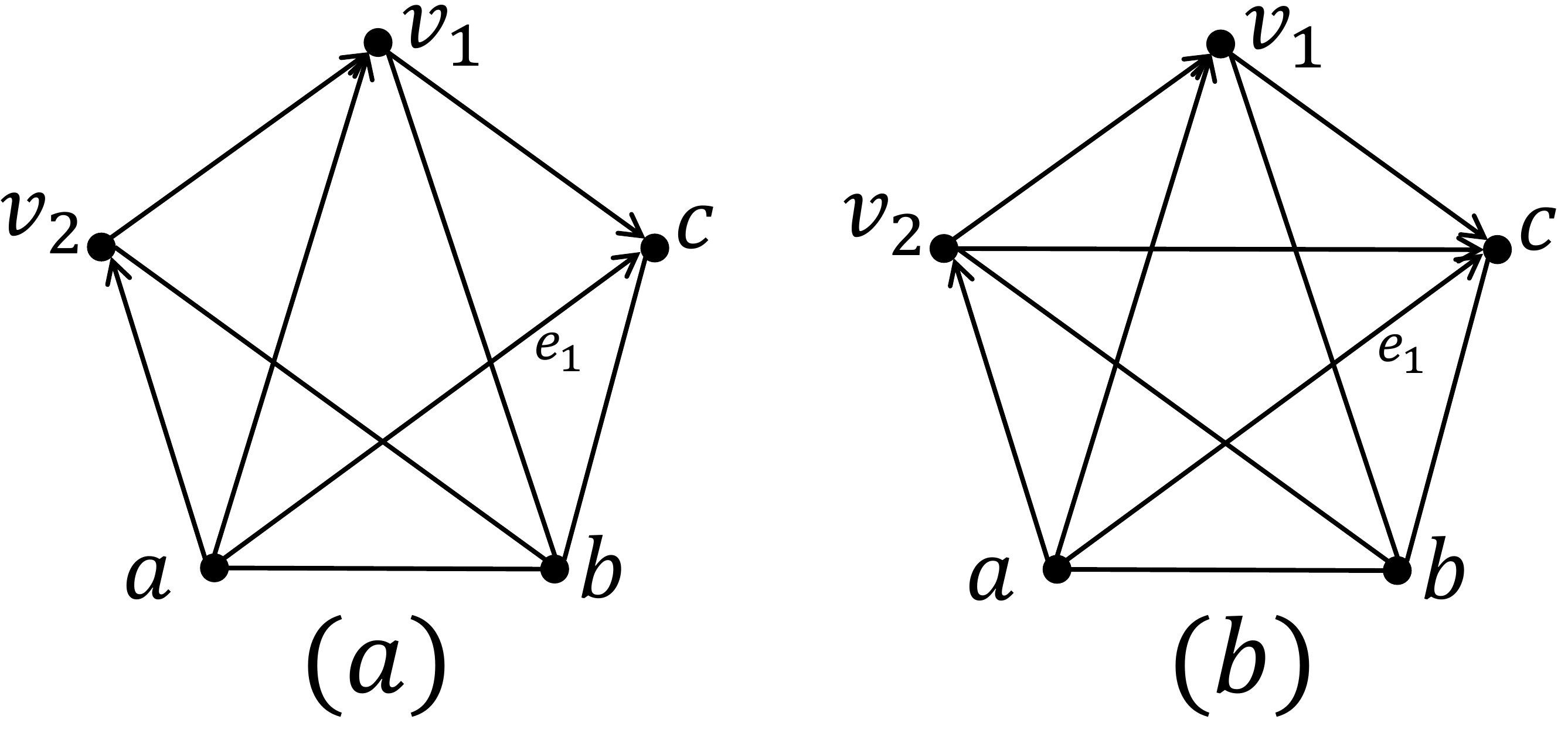}}
\caption{Step of the induction.}
\vspace{-7mm}
\label{fig:s10}
\end{center}
\end{figure}

We claim that this procedure always repeats, i.e., at step $i$, we end up with skeleton $K_i$, which contains $(v_j,c)$, $(a,v_j)$, $\{v_j,b\}$, for $j\in\{1,...,i\}$ and $(v_k,v_j)$, for $1\le j<k\le i$, with a triangle induced on vertices $\{v_i,b,a\}$, in which we have learned only the orientation of $(a,v_i)$. We prove this claim by induction. We have already proved the base of the induction above. For the step of the induction, suppose the hypothesis is true for $i-1$. Add vertex $v_i$ to form a structure of form $S_{10}$ for $(a,v_{i-1})$. $v_i$ should be adjacent to $v_j$, for $j\in\{1,...,i-2\}$; otherwise, using rule 4 on subgraph induced on vertices $\{v_i,v_{i-1},v_j,b\}$, we will also learn $(b,v_j)$. Moreover, using rule 2 on triangle induced on vertices $\{v_i,v_{i-1},v_j\}$, the direction of $\{v_i,v_j\}$ should be $(v_i,v_j)$. Also, using rule 4 on subgraph induced on vertices $\{v_i,v_{i-1},c,b\}$, we will also learn $(b,c)$. Therefore, we should have the edge $\{v_i,c\}$ too.

We showed that $S_0$ is a subgraph only if $S_{10}$ is a subgraph, and $S_{10}$ is a subgraph only if the structure in Figure \ref{fig:s10}(b) is a subgraph, and this chain of required subgraphs continue.
Therefore, since the order of the graph is finite, there exist a step where since we cannot add a new vertex, it is not possible to have one of the required subgraphs, and hence we conclude that $S_0$ is not a subgraph.\\
\end{itemize}

\noindent
\textit{Rule 3.}

Since edges $e_1$ and $e_2$ form a v-structure, they should appear in $A(\textit{Ess}(G^*))$ as well.  Therefore, we should have the condition of rule 3 satisfied when only intervening on $\mathcal{I}_1$ as well, which implies that $e\in R(A_{G^*}^{(\mathcal{I}_1)},G^*)$, which is a contradiction.\\

\noindent
\textit{Rule 4.}


Without loss of generality, assume $e_1\in R(A_{G^*}^{(\mathcal{I}_1)},G^*)\cup A(\textit{Ess}(G^*))$. Therefore, we should have the condition of rule 4 satisfied when only intervening on $\mathcal{I}_1$ as well, which implies that $e\in R(A_{G^*}^{(\mathcal{I}_1)},G^*)$, which is a contradiction.\\

The argument above proves that there is no edge $e$ such that $e\not\in R(A_{G^*}^{(\mathcal{I}_1)},G^*)$ and $e\not\in R(A_{G^*}^{(\mathcal{I}_2)},G^*)$, but $e\in R(R(A_{G^*}^{(\mathcal{I}_1)},G^*)\cup R(A_{G^*}^{(\mathcal{I}_2)},G^*),G^*)$.

\section{Proof of Theorem \ref{thm:app}}

Let $\mathcal{I}^*=\{v_1^*,...,v_k^*\}\in\arg\max_{\mathcal{I}:\mathcal{I}\subseteq V,|\mathcal{I}|= k} \mathcal{D}(\mathcal{I})$. We have
\begin{equation}
\begin{aligned}
\label{eq:app1}
&\mathcal{D}(\mathcal{I}^*)\overset{(a)}{\le}\mathcal{D}(\mathcal{I}^*\cup \mathcal{I}_i)=\mathcal{D}(\mathcal{I}_i)\\&+\sum_{j=1}^k[\mathcal{D}(\mathcal{I}_i\cup\{v_1^*,...,v_j^*\})-\mathcal{D}(\mathcal{I}_i\cup\{v_1^*,...,v_{j-1}^*\})]\\
&\overset{(b)}{\le}\mathcal{D}(\mathcal{I}_i)+\sum_{j=1}^k[\mathcal{D}(\mathcal{I}_i\cup\{v_j^*\})-\mathcal{D}(\mathcal{I}_i)],
\end{aligned}
\end{equation}
where $(a)$ follows from Lemma \ref{lem:mono}, and $(b)$ follows from Theorem \ref{thm:submodular}.
Define $\hat{\mathcal{D}}_{i,v,1}$ and $\hat{\mathcal{D}}_{i,v,2}$ as the first and second calls of subroutine in $i$-th step for variable $vv$, respectively. By the assumption of the theorem we have 
\begin{align*}
\mathcal{D}(\mathcal{I}_i\cup\{v_j^*\})-\epsilon\mathcal{D}(\mathcal{I}_i\cup\{v_j^*\})<\hat{\mathcal{D}}_{i,v^*_j,1}(\mathcal{I}_i\cup\{v_j^*\}),
\end{align*}
with probability larger than $1-\delta$. Therefore
\begin{align*}
\mathcal{D}(\mathcal{I}_i\cup\{v_j^*\})<\hat{\mathcal{D}}_{i,v^*_j,1}(\mathcal{I}_i\cup\{v_j^*\})+\epsilon\mathcal{D}(\mathcal{I}^*),
\end{align*}
with probability larger than $1-\delta$.
Similarly
\begin{align*}
\hat{\mathcal{D}}_{i,v^*_j,2}(\mathcal{I}_i)<\mathcal{D}(\mathcal{I}_i)+\epsilon\mathcal{D}(\mathcal{I}_i)\hspace{1cm}&w.p.>1-\delta,\\
\Rightarrow-\mathcal{D}(\mathcal{I}_i)< -\hat{\mathcal{D}}_{i,v^*_j,2}(\mathcal{I}_i)+\epsilon\mathcal{D}(\mathcal{I}^*)\hspace{1cm}&w.p.>1-\delta,
\end{align*}
Therefore,
\begin{equation}
\label{eq:app2}
\begin{aligned}
\mathcal{D}&(\mathcal{I}_i\cup\{v_j^*\})-\mathcal{D}(\mathcal{I}_i)<\hat{\mathcal{D}}_{i,v^*_j,1}(\mathcal{I}_i\cup\{v_j^*\})\\
&-\hat{\mathcal{D}}_{i,v^*_j,2}(\mathcal{I}_i)+2\epsilon\mathcal{D}(\mathcal{I}^*)\hspace{1cm}w.p.>1-2\delta.
\end{aligned}
\end{equation}
Also, by the definition of the greedy algorithm,
\begin{equation}
\label{eq:app3}
\begin{aligned}
\hat{\mathcal{D}}_{i,v^*_j,1}&(\mathcal{I}_i\cup\{v_j^*\})-\hat{\mathcal{D}}_{i,v^*_j,2}(\mathcal{I}_i)\\
&\le\hat{\mathcal{D}}_{i,v_{i+1},1}(\mathcal{I}_i\cup\{v_{i+1}\})-\hat{\mathcal{D}}_{i,v_{i+1},2}(\mathcal{I}_i)\\
&=\hat{\mathcal{D}}_{i,v_{i+1},1}(\mathcal{I}_{i+1})-\hat{\mathcal{D}}_{i,v_{i+1},2}(\mathcal{I}_i),
\end{aligned}
\end{equation}
and similar to \eqref{eq:app2}, we have
\begin{equation}
\label{eq:app4}
\begin{aligned}
\hat{\mathcal{D}}&_{i,v_{i+1},1}(\mathcal{I}_{i+1})-\hat{\mathcal{D}}_{i,v_{i+1},2}(\mathcal{I}_i)<
\mathcal{D}(\mathcal{I}_{i+1})\\
&-\mathcal{D}(\mathcal{I}_i)+2\epsilon\mathcal{D}(\mathcal{I}^*)\hspace{1cm}w.p.>1-2\delta.
\end{aligned}
\end{equation}
Therefore, from equations \eqref{eq:app2}, \eqref{eq:app3}, and \eqref{eq:app4} we have
\begin{equation}
\label{eq:app5}
\mathcal{D}(\mathcal{I}_i\cup\{v_j^*\})-\mathcal{D}(\mathcal{I}_i)<\mathcal{D}(\mathcal{I}_{i+1})-\mathcal{D}(\mathcal{I}_i)+4\epsilon\mathcal{D}(\mathcal{I}^*),
\end{equation}
with probability larger than $1-4\delta$. Plugging \eqref{eq:app5} back in \eqref{eq:app1}, we get
\begin{align*}
\mathcal{D}(\mathcal{I}^*)&<\mathcal{D}(\mathcal{I}_i)+\sum_{j=1}^k[\mathcal{D}(\mathcal{I}_{i+1})-\mathcal{D}(\mathcal{I}_i)+4\epsilon\mathcal{D}(\mathcal{I}^*)]\\
&=\mathcal{D}(\mathcal{I}_i)+k[\mathcal{D}(\mathcal{I}_{i+1})-\mathcal{D}(\mathcal{I}_i)]+4k\epsilon\mathcal{D}(\mathcal{I}^*),
\end{align*}
with probability larger than $1-4k\delta$. Therefore,
\begin{align*}
&\mathcal{D}(\mathcal{I}^*)-\mathcal{D}(\mathcal{I}_i)\\
&< k[\mathcal{D}(\mathcal{I}^*)-\mathcal{D}(\mathcal{I}_i)]-k[\mathcal{D}(\mathcal{I}^*)-\mathcal{D}(\mathcal{I}_{i+1})]+4k\epsilon\mathcal{D}(\mathcal{I}^*),
\end{align*}
with probability larger than $1-4k\delta$. Defining $a_i\coloneqq\mathcal{D}(\mathcal{I}^*)-\mathcal{D}(\mathcal{I}_i)$, and noting that $a_0=\mathcal{D}(\mathcal{I}^*)$, by induction we have 
\begin{align*}
a_k&=\mathcal{D}(\mathcal{I}^*)-\mathcal{D}(\mathcal{I}_k)\\
&<(1-\frac{1}{k})^k\mathcal{D}(\mathcal{I}^*)+4\epsilon\mathcal{D}(\mathcal{I}^*)\sum_{j=0}^{k-1}(1-\frac{1}{k})^j\\
&<[\frac{1}{e}+4\epsilon k]\mathcal{D}(\mathcal{I}^*)\hspace{1cm}w.p.>1-4k^2\delta.
\end{align*}
It concludes that
\[
\mathcal{D}(\mathcal{I}_k)>(1-\frac{1}{e}-4\epsilon k)\mathcal{D}(\mathcal{I}^*)\hspace{1cm}w.p.>1-4k^2\delta.
\]
Therefore, for $\epsilon=\frac{\epsilon'}{4k}$ and $\delta=\frac{\delta'}{4k^2}$,  Algorithms \ref{algorithm:GG} is a $(1-\frac{1}{e}-\epsilon')$-approximation algorithm with probability larger than $1-\delta'$.

\section{Proof of Theorem \ref{thm:comp}}

We run the algorithm for $k$ iterations. In each iteration, we execute the function $\hat{\mathcal{D}}(.)$ using Subroutine 1 for at most $n$ vertices. Furthermore, in this subroutine, we generate $N$ random DAGs by calling the function \textsc{RandEdge}, where in \cite{ghassami2018mec} it is shown that the complexity of each call is $O(n^{\Delta})$. Hence, the computational complexity of the algorithm is $O(knN\times n^{\Delta})$.

\section{Proof of Lemma \ref{lem:inMEC}}

We require the following lemma for the proof:
\begin{lemma}
\label{claim:tri}
A chordal graph has a directed cycle only if it has a directed cycle of size 3.
\end{lemma}
\begin{proof}
If the directed cycle is of size 3 itself, the claim is trivial. Suppose the cycle $C_n$ is of size $n>3$. Relabel the vertices of $C_n$ to have $C_n=(v_1,...,v_n,v_1)$. Since the graph is chordal, $C_n$ has a chord and hence we have a triangle on vertices $\{v_i,v_{i+1},v_{i+2}\}$ for some $i$. If the direction of $\{v_i,v_{i+2}\}$ is $(v_{i+2},v_i)$, we have the directed cycle $(v_i,v_{i+1},v_{i+2},v_i)$ which is of size 3. Otherwise, we have the directed cycle $C_{n-1}=(v_1,...,v_i,v_{i+2},..,v_n,v_1)$ on $n-1$ vertices. Relabeling the vertices from $1$ to $n-1$ and repeating the above reasoning concludes the lemma.
\end{proof}
\begin{proof}[Proof of Lemma \ref{lem:inMEC}]
All the components in the undirected subgraph of $Ess(G^*)$ are chordal \cite{hauser2012characterization}. Therefore, by Lemma \ref{claim:tri}, to insure that a generated directed graph is a DAG, it suffices to make sure that it does not have any directed cycles of length 3, which is one of the checks that we do in the proposed procedure. For checking if the generated DAG is in the same Markov equivalence class as $G^*$, it suffices to check if they have the same set of v-structures \cite{judea1991equivalence}, which is the other check that we do in the proposed procedure.
\end{proof}

\end{appendices}

%
%
%

\end{document}